%% file: main.tex
\documentclass[a4paper,10pt]{article}
\usepackage[utf8]{inputenc}
\usepackage[margin=1in]{geometry}

\usepackage{amsmath,amsfonts,amsthm,amssymb}
\usepackage{mathtools}
\usepackage{enumitem}
\usepackage{xspace}
\usepackage{bm}
\usepackage[extdef=true]{delimset}
\usepackage[colorlinks,allcolors=blue]{hyperref}
\usepackage[round]{natbib}
\usepackage[capitalize]{cleveref}
\usepackage{xspace}
\usepackage{enumitem}
\usepackage{booktabs}
\usepackage{adjustbox}
\usepackage{tablefootnote}
\usepackage{threeparttable}
\usepackage{caption}

\usepackage{thmtools,thm-restate}
\declaretheorem[parent=]{theorem}
\declaretheorem[parent=]{lemma}

\DeclarePairedDelimiter\ceil{\lceil}{\rceil}
\DeclarePairedDelimiter\floor{\lfloor}{\rfloor}
\newcommand{\ifrac}[2]{{#1}/{#2}}
\newcommand{\poly}{\operatorname{poly}}
\newcommand{\reals}{{\mathbb{R}}}
\newcommand{\E}{\mathbb{E}}
\newcommand{\condE}[2]{\E_{#2}\brk[s]*{#1}}

\newcommand{\dotp}{\boldsymbol{\cdot}}
\newcommand{\eqdef}{\triangleq}

\newcommand{\etade}{\hat{\eta}}
\newcommand{\wt}[1]{\widetilde{#1}}

\newcommand{\sm}{\beta}
\newcommand{\fbound}{F}
\newcommand{\fdiff}{\Delta}
\newcommand{\fdiffmax}{\bar{\Delta}}
\newcommand{\wbound}{D}
\newcommand{\wmax}{\bar{d}}
\newcommand{\wmaxeta}{\bar{d'}}
\newcommand{\wnorm}{d}
\newcommand{\initg}{\gamma}
\newcommand{\knowndiam}{\alpha}

\newcommand{\cF}{\mathcal{F}}

\newcommand{\affsigma}{\sigma}

\newcommand{\proj}[1]{\Pi_1\brk*{#1}}

\newcommand{\tO}{\smash{\wt{O}}}
\newcommand{\adasgd}[0]{\textsf{AdaSGD}\xspace}

\title{SGD with AdaGrad Stepsizes:\\
Full Adaptivity with High Probability
to Unknown Parameters, Unbounded Gradients and Affine Variance}
\author{%
    Amit Attia%
    \thanks{\scriptsize Blavatnik School of Computer Science, Tel Aviv University; \texttt{amitattia@mail.tau.ac.il}.}
    \and
    Tomer Koren%
    \thanks{\scriptsize Blavatnik School of Computer Science, Tel Aviv University, and Google Research Tel Aviv; \texttt{tkoren@tauex.tau.ac.il}.}
}

\begin{document}
\maketitle

\input{abstract.tex}
\input{introduction.tex}

\input{related_work.tex}

\input{preliminaries.tex}
\input{non_convex_setting.tex}

\input{convex_affine.tex}

\subsection*{Acknowledgements}
This work has received funding from the Israel Science Foundation (ISF, grant number 993/17),
Len Blavatnik and the Blavatnik Family foundation,
the Adelis Research Fund
for Artificial Intelligence,
the Yandex Initiative for Machine Learning at Tel Aviv University, and from the Tel Aviv University Center for AI and Data Science (TAD).

\bibliographystyle{abbrvnat}
\bibliography{references}

\appendix

\input{bounded_noise_reduction.tex}

\input{diameter_lower.tex}
\input{known-non-convex.tex}
\input{non_convex_appendix.tex}
\input{convex_appendix.tex}

\end{document}

%% file: abstract.tex
\begin{abstract}%
We study Stochastic Gradient Descent with AdaGrad stepsizes: a popular adaptive (self-tuning) method for first-order stochastic optimization.
Despite being well studied, existing analyses of this method suffer from various shortcomings: they either assume some knowledge of the problem parameters, impose strong global Lipschitz conditions, or fail to give bounds that hold with high probability.
We provide a comprehensive analysis of this basic method without any of these limitations, in both the convex and non-convex (smooth) cases, that additionally 
supports a general ``affine variance'' noise model and provides 
sharp rates of convergence in both the low-noise and high-noise~regimes.
\end{abstract}

%% file: introduction.tex
\section{Introduction}
Stochastic Gradient Descent (SGD) \citep{robbins1951stochastic} is the de facto standard optimization method in modern machine learning.
SGD updates take the following form; $$w_{t+1}=w_{t}-\eta_{t} g_t,$$ starting from an arbitrary point $w_{1}$, where $\eta_{t}$ is the step-size at time $t$ and $g_{t}$ is a stochastic gradient at $w_{t}$.
The sequence of stepsizes $\eta_1,\eta_2,\ldots$ is notoriously hard to tune in machine learning applications~\citep[e.g.,][]{bottou2012stochastic,schaul2013no}, especially when the problem parameters (e.g., smoothness or Lipschitz constant of the objective) or estimates thereof are unknown.
This has popularized the study of \emph{adaptive gradient methods}~\citep[e.g.,][]{duchi2011adaptive,kingma2014adam,reddi2018convergence,tran2019convergence,alacaoglu2020new}
where stepsizes are computed adaptively based on magnitudes of previous stochastic gradients.

Although their non-adaptive counterparts are well studied, optimization theory of adaptive methods is still lacking and attract much interest, particularly in the smooth and non-convex setting.
The different types of analyses performed by previous works in this setting revolve around the assumptions taken in order to establish convergence which includes the type of noise (bounded variance, sub-Gaussian, bounded support, affine variance, etc.), whether the true gradients or the stochastic gradients are uniformly bounded \citep{ward2019adagrad,kavis2022high}, and also whether to use a step-size that does not depends on the current stochastic gradient \citep{li2019convergence,li2020high} or to handle the bias introduced by such a dependency \citep{ward2019adagrad,Faw2022ThePO}.

We focus on perhaps the simplest and most well-studied adaptive optimization method:
SGD with AdaGrad-type stepsizes, which we refer to as~\adasgd (also known in the literature as AdaGrad-Norm; e.g.,
\citealt{ward2019adagrad,Faw2022ThePO}). 
For parameters $\eta,\initg > 0$, \adasgd performs SGD updates with stepsizes
\begin{align*}
    \eta_t = \frac{\eta}{\sqrt{\initg^2 + \sum_{s=1}^t \norm{g_s}^2}},
\end{align*}
where $g_t$ is a stochastic gradient at $w_t$.

Although numerous convergence results for smooth \adasgd exist in the literature (mostly in the non-convex case, except for \citet{li2019convergence} discussed below), they all fall short of at least one important aspect.
First, many of the existing analyses require a strong global Lipschitz assumption on the objective \citep[][]{ward2019adagrad,zhou2018convergence,kavis2022high} which does not hold even for simple quadratic functions in one dimension, and one that is not required for the analysis of analogous non-adaptive methods. 
Second, 
several results assume partial knowledge of the problem parameters, such as the smoothness parameter~\citep[e.g.,][]{li2019convergence,li2020high}, and thus in a sense do not fully establish the desired adaptivity of \adasgd.
A summary of the various results for SGD with AdaGrad stepsizes is presented in \cref{table:adasgd-comparison}.

To our knowledge, only \citet{Faw2022ThePO} do not impose neither of the above limiting assumptions, but their bounds do not hold with high probability: namely, they give convergence bounds with a polynomial dependence on the probability margin (granted, under a somewhat weaker tail assumption on the gradient noise).
On the other hand, their results apply in a significantly more challenging setting with an ``affine variance'' noise model, where the variance of the gradient noise is not uniformly bounded and may scale with the squared norm of the true gradient.

\begin{table*}[ht]
    \caption{\label{table:adasgd-comparison} Convergence results for non-convex SGD with AdaGrad stepsizes under different assumptions. For simplicity, the rate column consider the dependence in $T$ (number of steps) and $\sigma$ (noise variance),
    ignoring other parameters and logarithmic factors.
    The ``high probability'' column indicates results with poly-logarithmic (rather than polynomial) dependence in the probability margin.
    The ``convex guarantee'' column indicates an additional analysis for the convex case.
    }
    \adjustbox{width=1.\textwidth}{
    \centering
    \begin{threeparttable}
    \begin{tabular}{cccccccc}
    \toprule
    &
    \sc \begin{tabular}[c]{@{}c@{}} adaptive \\ stepsize \end{tabular} &
    \sc \begin{tabular}[c]{@{}c@{}} unknown\\ smoothness\end{tabular} &
    \sc \begin{tabular}[c]{@{}c@{}} high \\ probability\end{tabular} &
    \sc \begin{tabular}[c]{@{}c@{}} unbounded \\ gradients\end{tabular} &
    \sc \begin{tabular}[c]{@{}c@{}} affine \\ noise\end{tabular} &
    \sc \begin{tabular}[c]{@{}c@{}} convex \\ guarantee \end{tabular} &
    \sc non-convex rate \\
    \midrule
    \citet{ghadimi2013stochastic}\tnote{a} &
    & & \checkmark & \checkmark & & &
    $\tfrac{1 + \sigma^2}{T}+\tfrac{\sigma}{\sqrt{T}}$ \\
    \citet{bottou2018optimization} &
    & & & \checkmark & \checkmark & & 
    $\tfrac{1}{T}+\tfrac{\sigma}{\sqrt{T}}$ \\
    {\bf This paper} (Thm.\,\ref{thm:known-non-convex-convergence}) &
    & & \checkmark & \checkmark & \checkmark & \checkmark\tnote{b}& 
    $\tfrac{1+\sigma^2}{T}+\tfrac{\sigma}{\sqrt{T}}$ \\
    \midrule
    \citet{li2019convergence} &
    \checkmark & &  & \checkmark &  & \checkmark &
    $\tfrac{1}{T}+\tfrac{\sigma}{\sqrt{T}}$ \\
    \citet{ward2019adagrad} &
    \checkmark & \checkmark &  &  &  & &
    $\tfrac{1}{T}+\tfrac{\sigma(1+\sigma)}{\sqrt{T}}$ \\
    \citet{kavis2022high} &
    \checkmark & \checkmark & \checkmark &  &  & &
    $\tfrac{1+\sigma^4 }{T}+\tfrac{\sigma(1+\sigma^2)}{\sqrt{T}}$ \\
    \citet{Faw2022ThePO} &
    \checkmark & \checkmark & & \checkmark & \checkmark & &
    $\tfrac{1}{T}+\tfrac{\sigma(1+\sigma)}{\sqrt{T}}$ \\
    \citet{liu2023high}\tnote{c} &
    \checkmark & \checkmark & \checkmark & \checkmark & & &
    $\tfrac{1+\sigma^5}{T}+\tfrac{\sigma(1+\sigma^3)}{\sqrt{T}}$ \\
    {\bf This paper} (Thms.\,\ref{thm:non_convex_convergence} \& \ref{thm:convex_convergence_affine})
    &
    \checkmark & \checkmark & \checkmark & \checkmark & \checkmark & \checkmark &
    $\tfrac{1}{T}+\tfrac{\sigma(1+\sigma)}{\sqrt{T}}$ \\
    \bottomrule
    \end{tabular}
    \begin{tablenotes}\footnotesize
    \item[a] The high probability result is obtained by amplifying the success probability using multiple runs of SGD.
    \item[b] Though not formalized in this paper, a result for fixed-stepsize SGD in the convex case can be extracted from our analysis in the adaptive case.
    \item[c] This manuscript has appeared on arXiv after the publication of the present work.
    \end{tablenotes}
    \end{threeparttable}
    }
\end{table*}

\subsection{Our contributions}

We provide a comprehensive analysis of \adasgd that does not suffer from any of the former issues:
for a $\sm$-smooth objective and under a general affine variance noise model, we establish high probability convergence guarantees in both the non-convex and convex regimes, without any knowledge of the problem parameters and without imposing a global Lipschitz condition on the objective nor a bound on the distance to a minimizer.
Further, our analysis is considerably simpler than previous analyses of \adasgd in similar generality~\citep[e.g.,][]{Faw2022ThePO}.

In more detail, we consider a stochastic first-order model with affine variance noise, similar to the one of \citet{Faw2022ThePO}; given a point $w \in \reals^d$ we assume that the gradient oracle returns a random vector $g(w)$ such that
$
    \E[g(w) \mid w] = \nabla f(w)
$
and, with probability one, satisfies%
\footnote{While our analysis assumes for simplicity affine noise bounded with probability one, we provide a general reduction to allow for sub-Gaussian noise with bounded affine variance; see \cref{sec:sub_gaussian_reduction} for details.}
\begin{align*}
    \norm{g(w)-\nabla f(w)}^2 \leq \sigma_0^2 + \sigma_1^2 \norm{\nabla f(w)}^2
    ,
\end{align*}
where $\sigma_0,\sigma_1 \geq 0$ are noise parameters, unknown to the algorithm.
Our main results in this setup are as follows:
\begin{itemize}[leftmargin=*]
\item 
Our first main result establishes a high probability convergence guarantee in the general $\sm$-smooth (not necessarily convex) setting.
Given a $\sm$-smooth objective $f$, we show that for $T$-steps \adasgd, with probability at least $1-\delta$,
\begin{align*}%
\frac{1}{T} \! \sum_{t=1}^T \norm{\nabla f(w_t)}^2
&
\leq
C_1 \frac{\poly\brk{\log\tfrac{T}{\delta}}}{T}
+ C_2 \frac{\sigma_0 \poly\brk{\log\tfrac{T}{\delta}}}{\sqrt{T}}
,
\end{align*}
where $C_1$ and $C_2$ are quantities that depend polynomially on the problem parameters ($\sm,\sigma_0,\sigma_1,\eta,\gamma$; the dependence on $1/\gamma$ is in fact logarithmic). 
The formal statement is given in \cref{thm:non_convex_convergence}.
\item
Our second main result is a convergence guarantee for \adasgd in the same unconstrained optimization setup but with an additional convexity assumption.
Given a $\sm$-smooth and \emph{convex} objective $f$, we show that for $T$-steps \adasgd, with probability at least $1-\delta$, it holds that
\begin{align*}%
    f \brk3{\frac1T \! \sum_{t=1}^{T} w_t \!} \!-\! f^\star \!
    &
    \leq
    C_3 \frac{\poly\brk{\log\tfrac{T}{\delta}}}{T}
    + C_4 \frac{\sigma_0 \poly\brk{\log\tfrac{T}{\delta}}}{\sqrt{T}}
    ,
\end{align*}
where $f^\star = \min_{w \in \reals^d} f(w)$ and $C_3$ and $C_4$ are again quantities that depend polynomially on the problem parameters ($\sm,\sigma_0,\sigma_1,\eta,\gamma$; here the dependence on $1/\gamma$ is polynomial).
The formal statement is provided in \cref{thm:convex_convergence_affine}.
\end{itemize}

These convergence bounds match the rates of the form  $O(1/T+\sigma_0/\sqrt{T})$ for \emph{non-adaptive, properly tuned SGD} in the convex~\citep{lan2012optimal} and non-convex cases~\citep{ghadimi2013stochastic} cases, up to problem dependent constants and logarithmic factors.
However, we emphasize again that our results hold true, with high probability, \emph{without} assuming uniformly bounded gradients and \emph{without} requiring any knowledge of the problem parameters ($\sm,\sigma_0,\sigma_1$).
To our knowledge, in the non-convex case our result is the first convergence guarantee to hold with high probability (with a proper poly-log dependence on the probability margin) without requiring a global Lipschitz assumption---and in fact, it holds even in a more general affine variance noise model.
Similarly, as far as we are aware, our result in the convex case constitutes the first (high probability) analysis of adaptive SGD updates that applies in an unconstrained optimization setup, i.e., without restricting the updates to a domain of bounded (and a-priori known) diameter using projections.

Further, our bounds maintain a remarkable property of the non-adaptive bounds for smooth SGD \citep{lan2012optimal,ghadimi2013stochastic}:
as can be seen from the bounds above, 
they exhibit optimal rates (in the context of non-accelerated gradient methods) in both the high-noise and low-noise regimes, in the sense that they improve from the traditional stochastic $1/\sqrt{T}$ rate to a fast $1/T$ rate, up to logarithmic factors and problem-dependent constants, when the additive noise parameter $\sigma_0$ is sufficiently small: namely, when $\sigma_0 = O(1/\sqrt{T})$.
In contrast to the non-adaptive bounds, however, this interpolation between different rates occurs automatically, without knowledge of the underlying parameters and their relevant regime.
Unlike prior work on the affine noise model, our rate-interpolation takes place even when the affine term is large, i.e., when $\sigma_1 \gg 0$; for comparison, \citet{Faw2022ThePO} required that $\sigma_1 = O(1/\sqrt{T})$ for drawing a similar conclusion.
For completeness, we also provide a high probability analysis of non-adaptive, smooth SGD (with known parameters) under the affine noise model, complementing the result in expectation of \citet[Theorem 4.8]{bottou2018optimization}; see details in \cref{sec:known-non-convex}.

Finally, it is worth highlighting another remarkable adaptivity property of \adasgd, demonstrated by our results: the algorithm is able to automatically adapt to the optimization problem at hand and the same step-size schedule simultaneously achieves the optimal rates of convergence of (non-accelerated) SGD in both the convex and non-convex smooth (even unconstrained) optimization settings.

\subsection{Analysis overview}

Let us highlight some of the key challenges in analyzing SGD with adaptive stepsizes and outline our approach for overcoming them.
We begin with the non-convex smooth case, which has been the focus of prior work.

\paragraph{Non-convex regret analysis.}

The starting point of our analysis is a ``non-convex regret bound'' for \adasgd, similar to~\citet{kavis2022high}; 
for a $\sm$-smooth objective $f$ it can be shown (see \cref{lem:non-convex-regret} for details and a proof) that
\begin{align*}
    \sum_{t=1}^T \nabla f(w_t) \dotp g_t
    &\leq
    \frac{\initg \fdiff_1}{\eta}
    + \brk*{\frac{2 \fdiffmax_T}{\eta}+\eta \sm} \sqrt{\sum_{t=1}^T \norm{g_t}^2}
    ,
\end{align*}
where $\fdiff_1=f(w_1)-f^\star$ and $\fdiffmax_T = \max_{t \leq T} f(w_t)-f^\star$.

The challenging term to handle in this analysis turns out to be the one that involves $\fdiffmax_T$.
Indeed, assume for example a deterministic bound $\fdiffmax_T \leq \fbound$ (e.g., this is the case whenever $f$ is upper bounded uniformly); taking  expectations of both sides, we obtain
\begin{align*}
    \sum_{t=1}^T & \E\norm{\nabla f(w_t)}^2
    \leq
    \frac{\initg \fdiff_1}{\eta}
    + \brk*{\frac{2 \fbound}{\eta}+\eta \sm} \sqrt{\sum_{t=1}^T \E\norm{g_t}^2}
    .
\end{align*}
Then, using $\norm{g_t}^2 \leq 2 \norm{\nabla f(w_t)-g_t}^2 + 2 \norm{\nabla f(w_t)}^2$ and our assumption on the noise yields the bound
\begin{align*}
    \E\brk[s]3{ \frac1T \sum_{t=1}^T \norm{\nabla f(w_t)}^2}
    &
    \leq
    \frac{2 \initg \fdiff_1}{\eta T}
    + (1+\sigma_1^2) \brk*{\frac{2 \fbound}{\eta}+\eta \sm}^2 \frac1T
    + \brk*{\frac{2 \fbound}{\eta}+\eta \sm} \sigma_0 \sqrt{\frac8T}
    ,
\end{align*}
which concludes an in-expectation analysis.

\paragraph*{Bounding $\bm{\fdiffmax_T}$ w.h.p.}
The issue, however, is that $\fdiffmax_T$ is a random variable which is not bounded a-priori. 
Due to the probabilistic dependence between the different factors in the regret bound, our approach proceeds by establishing a bound over $\fdiffmax_T$ that holds with high probability (rather than in expectation).
Due to smoothness, we can bound $\fdiffmax_T$ via
\begin{align*}
    f (w_{t+1})
    &\leq
    f(w_{1})
    - \sum_{s=1}^{t} \eta_s \nabla f(w_s) \dotp g_s
    + \frac{\sm}{2} \sum_{s=1}^{t} \eta_s^2 \norm{g_s}^2
    .
\end{align*}
The second summation in the bound is straightforward to handle and can be bounded (deterministically) by a logarithmic term (see \cref{lem:sum_eta_g_t_2_affine} for details).
The first summation is more problematic, again due to (conditional) probabilistic dependencies involving the stepsize $\eta_s$, which is crucially a random variable that depends on stochastic gradient $g_s$.

\paragraph{Decorrelated stepsizes.}

To break the correlation between $\eta_{s}$ and $g_{s}$ in the second term above, \citet{kavis2022high} replace $\eta_s$ with $\eta_{s-1}$.
Then, for bounding the difference between $\eta_s$ and $\eta_{s-1}$ they assumed a \emph{global Lipschitz assumption} which limits the effect of each individual gradient on the stepsize $\eta_s$. 
To avoid this limiting assumption, we use a different ``decorrelated'' proxy for the stepsize, defined by replacing $\norm{g_t}^2$ with $(1+\sigma_1^2)\norm{\nabla f(w_t)}^2 + \sigma_0^2$;
this was introduced first by \citet{ward2019adagrad} and then modified by \citet{Faw2022ThePO} to allow for affine-variance noise.
We can then decompose using the proxy,
\begin{align*}
    - \eta_s \nabla f(w_s) \dotp g_s
    &
    =
    - \tilde{\eta}_s \norm{\nabla f(w_s)}^2
    + \tilde{\eta}_s \nabla f(w_s) \dotp (\nabla f(w_s) - g_s)
    + (\tilde{\eta}_s-\eta_s) \nabla f(w_s) \dotp g_s
    ,
\end{align*}
and bound the cumulative sum of each of the terms.

\paragraph*{Concentration with affine variance.}
The challenging term to bound is the martingale
\begin{align*}
    \sum_{s=1}^t \tilde{\eta}_s \nabla f(w_s) \dotp (\nabla f(w_s) - g_s).
\end{align*}
Bounding the martingale is challenging because the magnitude $\norm{\nabla f(w_s)-g_s}$ depends on the previous steps with affine-variance noise.
The key observation is that since we wish to establish a bound for $f(w_{t+1})-f^\star$ for all $t$, and thanks to the well-known property of smooth functions that $\norm{\nabla f(w)}^2 \leq 2 \sm (f(w)-f^\star)$, it is suffice to perform an inductive argument and use the bound for the previous steps to bound the affine noise terms.
As the noise bound in a given step is in itself a random variable that depends on previous stochastic gradients, we appeal to a Bernstein-type martingale concentration inequality (specifically, we use one due to \citealt{li2020high}) and a careful analysis yields a bound on $- \sum_{s=1}^t \eta_s \nabla f(w_s) \dotp g_s$ and establishes a bound on $\fdiffmax_T$ without the need for a global Lipschitz assumption.

\paragraph*{Additional challenges in the convex case.}

We follow a similar approach using a regret bound and a bound of $\norm{w_t-w^\star}^2$ which holds for all $t \leq T$ with high probability.
In the convex case we need to perform a similar decomposition to bound $- \sum_{s=1}^t \eta_s g_s \dotp (w_s-w^\star)$ where $w^\star$ is a minimizer of $f$.
Using the proxy stepsize sequence $\tilde{\eta}_t$
does not work.
Instead, we employ a slightly different decorrelated stepsize sequence $(\etade_s)_s$
that (A) suffices for bounding $\sum_{s=1}^t (\etade_s-\eta_s) g_s \dotp (w_s-w^\star)$ and (B) keeps the norm of the zero-mean vector $\etade_s (\nabla f(w_s) - g_s)$ bounded with probability one, which is crucial for the martingale analysis.
The martingale lemma of \citet{li2020high} is too loose for bounding $\sum_{s=1}^t \etade_s (\nabla f(w_s) - g_s) \dotp (w_s-w^\star)$ and we rely instead on the martingale analysis technique of \citet{carmon2022making} which uses an empirical Bernstein inequality introduced by \citet{howard2021time}.
The empirical Bernstein is necessary in order to tightly bound the martingale
and establish the bound on $\norm{w_t-w^\star}^2$.

%% file: related_work.tex
\subsection{Additional related work}

\paragraph*{SGD for non-convex smooth optimization.}
The first to analyze SGD in the smooth non-convex setting and obtain tight convergence bounds were \citet{ghadimi2013stochastic}. They established that a \emph{properly tuned} $T$-steps SGD obtain a rate of $O\brk{\ifrac{1}{T}+\ifrac{\sigma}{\sqrt{T}}}$ assuming uniformly bounded variance of $\sigma^2$. A tight lower bound is given by \citet{arjevani2022lower}.
\citet[Theorem 4.8]{bottou2018optimization} extended the convergence result for the affine variance noise model.

\paragraph*{Adaptive methods for non-convex smooth optimization.}
An extensive line of work tackles convergence guarantees in expectation.%
\footnote{We include here results with polynomial dependence in the probability margin as they can be derived using Markov's inequality.}
\citet{ward2019adagrad} gave a rate of $\tO\brk{\ifrac{1}{\sqrt{T}}}$ for $T$-steps Adaptive SGD, assuming uniformly bounded gradients.
In parallel, \citet{li2019convergence} obtained a rate of $\tO\brk{\ifrac{1}{\sqrt{T}}}$ without assuming uniformly bounded gradients, assuming instead knowledge of the smoothness parameter and using a step-size that does not depend on the current stochastic gradient.
\citet{li2019convergence} also obtain a convergence guarantee for the smooth and convex case, again assuming knowledge of the smoothness.
A recent work by \citet{Faw2022ThePO} assumed affine variance noise and achieved the desired $\tO\brk{\ifrac{1}{\sqrt{T}}}$ in the noisy regime without assuming uniformly bounded gradients.

\paragraph*{High probability convergence bounds.}
Given knowledge of the smoothness and sub-Gaussian noise bounds, \citet{ghadimi2013stochastic} established convergence with high probability for non-adaptive SGD with properly tuned step-size.
Few high probability results of adaptive methods exist.
\citet{zhou2018convergence} and \citet{li2020high} established a high probability convergence result for AdaGrad and delayed AdaGrad respectively, assuming sub-Gaussian noise and knowledge of the smoothness to tune the step-size. \citet{kavis2022high} obtained a high probability convergence for Adaptive SGD, assuming uniformly bounded gradients and either sub-Gaussian noise or uniformly bounded stochastic gradients.
Assuming heavy-tailed noise, \citet{Cutkosky2021HighprobabilityBF} obtained convergence with high probability for normalized SGD with clipping and momentum with a tuned step-size according to a moment bound of the noise.

\paragraph{Unconstrained stochastic convex optimization.}

Also related to our work are \citet{harvey2019tight} and \citet{carmon2022making} who focused on the (non-smooth) unconstrained and convex setting.
\citet{harvey2019tight} established convergence with high probability of SGD using a generalized freedman's inequality.
\citet{carmon2022making} devised a stochastic parameter-free method which obtain high probability convergence guarantee and their method also obtain optimal rate for \emph{noiseless} convex and smooth objective.
Both papers used a careful martingale analysis in order to bound the distance of the iterates from a minimizer, a desired property for unconstrained stochastic optimization which we also establish for \adasgd.

\paragraph{Parameter-free methods for online optimization.}

The closely related literature on online optimization includes a plethora of adaptive methods.
Recent work in this context is mostly concerned with adaptivity to the norm of the comparator~\citep[e.g.,][]{orabona2016coin,cutkosky2018black,mhammedi2020lipschitz} and to the scale of the objective~\citep[e.g.,][]{orabona2018scale,mhammedi2020lipschitz}.

%% file: preliminaries.tex
\section{Preliminaries}

We consider unconstrained smooth optimization in $d$-dimensional Euclidean space $\reals^d$ with the $\ell_2$ norm, denoted $\norm{\cdot}$.
A function $f$ is said to be $\sm$-smooth if for all $x,y \in \reals^d$ it holds that $\norm{\nabla f(x)-\nabla f(y)} \leq \sm \norm{x-y}$.
This implies, in particular, that for all $x,y \in \reals^d$, the following bound is true: $f(y) \leq f(x) + \nabla f(x) \dotp (y-x)+\frac12 \sm\norm{y-x}^2$.
We additionally assume that $f$ is lower bounded over $\reals^d$, and admits a minimum $f^\star = \min_{w \in \reals^d} f(w)$.
We assume a standard stochastic first-order model, in which we are provided with an oracle $g(w)$ where, given a point $w \in \reals^d$ returns a random vector $g(w) \in \reals^d$ such that $\E[g(w) \mid w] = \nabla f(w)$.

\paragraph*{Affine noise model.}
We assume a general noise model, referred to in the literature as the ``affine variance noise'' (e.g., \citet{Faw2022ThePO} and references therein).
In this model, the gradient oracle $g$ satisfies 
\begin{align*}
    \forall ~ w \in \reals^d :
    \quad
    \norm{g(w)-\nabla f(w)}^2 \leq \sigma_0^2 + \sigma_1^2 \norm{\nabla f(w)}^2
\end{align*}
(deterministically), where $\sigma_0,\sigma_1 \geq 0$ are noise parameters.
For simplicity, we will assume throughout a deterministic bound on the noise magnitude, as stated above; in \cref{sec:sub_gaussian_reduction}, through a generic reduction, we extend our arguments to apply in a more general sub-Gaussian noise model with affine variance.

The affine noise model permits the noise magnitude to grow as the norm of the true gradient increases (without the affine term, i.e., assuming $\sigma_1=0$, the gradients are effectively ``noiseless'' for gradients of large magnitude).
We refer to \citet{Faw2022ThePO} for an additional discussion of this model.

\paragraph*{Adaptive stepsize Stochastic Gradient Descent (\adasgd).}
Starting from an arbitrary point $w_1 \in \reals^d$, for parameters $\eta,\initg > 0$, the update rule of \adasgd is
\begin{align*} %
    w_{t+1}
    &=
    w_{t}-\eta_t g_t,
\end{align*}
where $g_t = g(w_t)$ is a stochastic gradient at $w_t$ and the step-size $\eta_t$ is given by
\begin{align}\label{eq:adaptive-eta}
    \forall ~ 1 \leq t \leq T :
    \qquad
    \eta_t = \frac{\eta}{\sqrt{\initg^2 + \sum_{s=1}^t \norm{g_s}^2}}
    .
\end{align}
\paragraph*{Additional notation.}
Throughout, we use $\condE{\cdot}{t-1} \eqdef \E\brk[s]{\cdot \mid g_1,\ldots,g_{t-1}}$ to denote the conditional expectation with respect to the stochastic gradients queried up to (not including) step $t$.  We use $\log$ to denote the base $2$ logarithm. 

%% file: non_convex_setting.tex
\section{Analysis in the general non-convex case}
\label{sec:non-convex}

First, we focus on the general smooth case without convexity assumptions.
Our main result of this section is the following convergence bound for \adasgd.
\begin{theorem}\label{thm:non_convex_convergence}
Assume that $f$ is $\sm$-smooth.
Then for \adasgd, for any $\delta \in (0,1/3)$,
it holds with probability at least $1-\delta$ that
\begin{enumerate}[label=(\roman*)]
    \item%
    $\begin{aligned}[t]
        \frac{1}{T} \sum_{t=1}^T \norm{\nabla f(w_t)}^2
        &
        \leq
        O\brk*{ \tfrac{2\fbound}{\eta} + \eta\sm }
        \frac{\sigma_0}{\sqrt{T}}
        + O\brk*{ \tfrac{\initg}{\eta} \fdiff_1 + (1 + \sigma_1^2 \log\tfrac{1}{\delta})\brk!{\tfrac{2\fbound}{\eta} + \eta\sm}^2 + \sigma_0^2 \log\tfrac{1}{\delta} } \frac{1}{T}
        ;
    \end{aligned}$
    \item%
    and $f(w_t)-f^\star \leq \fbound$ for all $1 \leq t \leq T+1$,
\end{enumerate}
where $\fdiff_1=f(w_1)-f^\star$ and 
\begin{align}
\label{eq:fbound}
\fbound
&=
\begin{aligned}[t]
    2 \fdiff_1
    + \brk!{3 \log \tfrac{T}{\delta} + 4 C_1} \eta \sigma_0
    &
    + \brk!{9 \log^2 \tfrac{T}{\delta} + 16 C_1^2} \eta^2 \sm \sigma_1^2
    + \eta^2 \sm C_1
    ,
\end{aligned}
\\
\label{eq:C1}
C_1
&=
\begin{aligned}[t]
    \log
    &
    \bigg(
        1
        \!+\!
        \frac{2 \sigma_0^2 T + 8 (1+\sigma_1^2) (\eta^2 \sm^2 T^3 + \sm \fdiff_1 T)}{\initg^2}
    \bigg)
    .\!
\end{aligned}
\end{align}
\end{theorem}
As stated, the bound in the theorem holds for bounded affine noise; for the precise bound in the more general sub-Gaussian case, see \cref{thm:non_convex_convergence_sub_gaussian} in \cref{sec:sub_gaussian_reduction}.

In addition to the features of this convergence result already highlighted in the introduction, we also note the high probability guarantee it provides that the maximal optimality gap $f(w_t)-f^\star$ remains $\tO(1)$ along the trajectory of \adasgd.
Comparing the result to the high probability analysis when the parameters are known (\cref{sec:known-non-convex}), the main difference is the bound of $f(w_t)-f^\star$ which is $O(\fdiff_1)$ when the parameters are known (for large enough $T$), while the bound of \cref{thm:non_convex_convergence} has additional terms that depend on $\sigma_0$ and $\sigma_1$.
We remark that the prior results of~\citet{Faw2022ThePO} in this setting include a somewhat worse high-order noise term, of order $\ifrac{\sigma_1(\sigma_0^2+\sigma_1^{12})}{\sqrt{T}}$ in the noisy regime (for $\sigma_1 = \omega(1)$).

\subsection{Proof of~\texorpdfstring{\cref{thm:non_convex_convergence}}{Theorem 1}}

Henceforth we use the following notations: let $\fdiff_t \eqdef f(w_t)-f^\star$ and $\fdiffmax_t \eqdef \max_{s \leq t} f(w_s)-f^\star$, for all $t$. 
We also denote $G_t \eqdef \brk{\initg^2 + \sum_{s=1}^t \norm{g_s}^2}^{1/2}$, thus, $\eta_t = \eta / G_t$.
As $\eta_t$ and $g_t$ are correlated random variables (indeed, $\eta_t$ depends on the norm of $g_t$), we define the following ``decorrelated stepsizes,''
\begin{align*}
    \forall ~ 1 \leq t \leq T :
    \tilde{\eta}_t \eqdef \frac{\eta}{\sqrt{G_{t-1}^2+ (1+\sigma_1^2)\norm{\nabla f(w_t)}^2 + \sigma_0^2}}
    .
\end{align*}
We begin by stating several lemmas that we use to prove \cref{thm:non_convex_convergence} (proofs are given later in the section).
The first lemma states a ``non-convex regret bound.''
\begin{lemma} \label{lem:non-convex-regret}
    If $f$ is $\sm$-smooth then for \adasgd we have
    \begin{align*}
        \sum_{t=1}^T \nabla f(w_t) \dotp g_t
        \leq
        \frac{\initg \fdiff_1}{\eta}
        + \brk*{\frac{2 \fdiffmax_T}{\eta}+\eta \sm} \sqrt{\sum_{t=1}^T \norm{g_t}^2}
        .
    \end{align*}
\end{lemma}
Note that the $\fdiffmax_T$ term in the right hand-side of the bound is a random variable.
In the noiseless setting with known smoothness and properly tuned stepsize, we know that $\fdiffmax_T = \fdiff_1$ since $\fdiff_{t+1} \leq \fdiff_{t}$ for all $t$ and proving convergence would be much simplified.
The next lemma bounds $\fdiffmax_T = \max_{s \leq t} f(w_s)-f^\star$ with high probability in the noisy case.
\begin{lemma}\label{lem:fbound_affine}
With probability at least $1-\delta$, it holds that 
$
    \fdiffmax_{T+1}
    \leq
    \fbound
$
(with $F$ defined in \cref{eq:fbound}).
\end{lemma}
In order to translate the regret analysis of \cref{lem:non-convex-regret} to a high probability bound we need a bound on the sum of differences between $\norm{\nabla f(w_t)}^2$ and $g_t \dotp \nabla f(w_t)$ which we obtain in the following lemma (see \cref{sec:non_convex_appendix} for a proof).
\begin{lemma}\label{lem:basic_martingale_affine}
    with probability at least $1 - 2 \delta$, 
    it holds that
    \begin{align*}
        \sum_{t=1}^T \nabla f(w_t) \dotp (\nabla f(w_t)-g_t)
        &
        \leq
        \frac{1}{4} \sum_{t=1}^T \norm{\nabla f(w_t)}^2
        + 3 (\sigma_0^2+2 \sm \sigma_1^2 \fbound) \log \tfrac{1}{\delta}
        ,
    \end{align*}
    with $\fbound$ given in \cref{eq:fbound}.
\end{lemma}
Using the lemmas above we can now prove our main result.
\begin{proof}[Proof of \cref{thm:non_convex_convergence}]
All arguments below follow given the high probability events of \cref{lem:fbound_affine,lem:basic_martingale_affine}.
Hence, using a union bound, they follow with probability at least $1-3 \delta$.
We denote $C_2 = \ifrac{2 \fbound}{\eta}+\eta \sm$ to simplify notation.
From \cref{lem:non-convex-regret,lem:fbound_affine},
\begin{align*}
    \sum_{t=1}^T \nabla f(w_t) \dotp g_t
    \leq
    \frac{\initg \fdiff_1}{\eta}
    + C_2 \sqrt{\sum_{t=1}^T \norm{g_t}^2}
    .
\end{align*}
Observe that
\begin{align*}
    \norm{g_t}^2
    &\leq
    2\norm{\nabla f(w_t)}^2 + 2\norm{g_t - \nabla f(w_t)}^2
    \leq
    2(1+\sigma_1^2)\norm{\nabla f(w_t)}^2 + 2\sigma_0^2
    .
\end{align*}
Hence, using $\sqrt{a+b} \leq \sqrt{a} + \sqrt{b}$ for $a,b \geq 0$,
\begin{align*}
    \sqrt{\sum_{t=1}^T \norm{g_t}^2}
    &\leq
    \sqrt{2\sum_{t=1}^T (1+\sigma_1^2) \norm{\nabla f(w_t)}^2} + \sigma_0 \sqrt{2 T}
    .
\end{align*}
By the elementary fact $ab \leq a^2/2 + b^2/2$ with
$a=\sqrt{\tfrac12 \sum_{t=1}^T \norm{\nabla f(w_t)}^2}$ and $b=2 C_2 \sqrt{1+\sigma_1^2}$,
\begin{align*}
    C_2 \sqrt{2(1+\sigma_1^2)\sum_{t=1}^T \norm{\nabla f(w_t)}^2}
    &\leq
    \frac{1}{4} \sum_{t=1}^T \norm{\nabla f(w_t)}^2
    + 2(1+\sigma_1^2) C_2^2
    .
\end{align*}
Hence,
\begin{align*}
    \sum_{t=1}^T \nabla f(w_t) \dotp g_t
    &\leq
    \frac{\initg \fdiff_1}{\eta}
    + 2(1+\sigma_1^2) C_2^2
    + C_2 \sigma_0 \sqrt{2 T}
    + \frac{1}{4} \sum_{t=1}^T \norm{\nabla f(w_t)}^2
    .
\end{align*}
Summing the above with \cref{lem:basic_martingale_affine} and rearranging,
\begin{align*}
    \sum_{t=1}^T \norm{\nabla f(w_t)}^2
    &\leq
    \frac{2 \initg \fdiff_1}{\eta}
    + 4(1+\sigma_1^2) C_2^2
    + 6(\sigma_0^2 + 2\sm \sigma_1^2 \fbound) \log \tfrac{1}{\delta}
    + C_2 \sigma_0 \sqrt{8 T}
    .
\end{align*}
Substituting $C_2$ and dividing through by $T$ gives the theorem (due to the $O$ notation in the theorem we state a $1-\delta$ probability guarantee instead of $1-3\delta$).
\end{proof}
\subsection{Proof of \texorpdfstring{\cref{lem:non-convex-regret}}{Lemma 1}}
Before proving \cref{lem:non-convex-regret} we state a standard lemma (for completeness, we give a proof in \cref{sec:non_convex_appendix}).
\begin{lemma} \label{lem:sum_eta_g_t}
    Let $g_1,\ldots,g_T \in \reals^d$ be an arbitrary sequence of vectors, and let $G_0 > 0$. For all $t \geq 1$, define
    \begin{align*}
        G_t
        =
        \sqrt{G_0^2 + \sum_{s=1}^t \norm{g_s}^2}
        .
    \end{align*}
    Then
    \begin{align*}
        \sum_{t=1}^T \frac{\norm{g_t}^2}{G_t}
        \leq
        2 \sqrt{\sum_{t=1}^T \norm{g_t}^2}
        ,
        \qquad
        \text{and}
        \qquad
        \sum_{t=1}^T \frac{\norm{g_t}^2}{G_t^2}
        \leq
        2\log\frac{G_T}{G_0}
        .
    \end{align*}
\end{lemma}
\begin{proof}[Proof of \cref{lem:non-convex-regret}]
    Using $\sm$-smoothness, we have for all $t \geq 1$ that
    \begin{align*}
        f(w_{t+1})
        \leq
        f(w_t) - \eta_t \nabla f(w_t) \dotp g_t + \tfrac12 \sm \eta_t^2 \norm{g_t}^2
        ,
    \end{align*}
    which implies (with $\fdiff_t=f(w_t)-f^\star$)
    \begin{align*}
        \nabla f(w_t) \dotp g_t
        &\leq
        \frac{1}{\eta_t} \brk!{\fdiff_t - \fdiff_{t+1}} 
        + \tfrac12 \sm \eta_t \norm{g_t}^2
        .
    \end{align*}
    Summing the above over $t=1,\ldots,T$, we obtain 
    \begin{align*}
        \sum_{t=1}^T \! \nabla f(w_t) \dotp g_t
        &\leq
        \frac{\fdiff_1}{\eta_0}
        + \!\sum_{t=1}^T \brk3{ \frac{1}{\eta_t} 
        - \frac{1}{\eta_{t-1}} } \fdiff_t
        + \frac{\sm}{2} \!\sum_{t=1}^T\! \eta_t \norm{g_t}^2
        \!.
    \end{align*}
    Now, since $\eta_t \leq \eta_{t-1}$, we have for all $t$ that
    \begin{align*}
        &
        \frac{1}{\eta_t} - \frac{1}{\eta_{t-1}}
        \leq
        \eta_t \brk3{ \frac{1}{\eta_t^2} - \frac{1}{\eta_{t-1}^2} } 
        =
        \frac{\eta_t}{\eta^2} \brk3{\initg^2 + \sum_{s=1}^t \norm{g_s}^2 - \initg^2 - \sum_{s=1}^{t-1} \norm{g_s}^2}
        =
        \frac{\eta_t}{\eta^2} \norm{g_t}^2 
        ,
    \end{align*}
    thus
    \begin{align*}
        \sum_{t=1}^T \nabla f(w_t) \dotp g_t
        &\leq
        \frac{\fdiff_1}{\eta_0}
        + \sum_{t=1}^T \eta_t \norm{g_t}^2 \brk*{\frac{\fdiff_t}{\eta^2}+\frac{\sm}{2}}
        .
    \end{align*}
    Bounding $\fdiff_t \leq \fdiffmax_T$ and applying \cref{lem:sum_eta_g_t} on the second sum
    yields
    \begin{align*}
        \sum_{t=1}^T \nabla f(w_t) \dotp g_t
        &\leq
        \frac{\fdiff_1}{\eta_0}
        + \brk*{\frac{2 \fdiffmax_T}{\eta}+\eta \sm} \sqrt{\sum_{t=1}^T \norm{g_t}^2}
        .
    \end{align*}
    Plugging in the expressions for $\eta_0$ yields the lemma.
\end{proof}

\subsection{Proof of \texorpdfstring{\cref{lem:fbound_affine}}{Lemma 2}}
For our high probability analysis we use the following martingale concentration inequality~\citep[][Lemma~1]{li2020high}.
\begin{lemma}
	\label{lem:sub_gaussian}
	Assume that $Z_1, Z_2, ..., Z_T$ is a martingale difference sequence with respect to $\xi_1, \xi_2, ..., \xi_T$ and $\E_t \left[\exp(Z_t^2/\sigma_t^2)\right] \leq \exp(1)$ for all $t$, where $\sigma_1,\ldots,\sigma_T$ is a sequence of random variables such that $\sigma_t$ is measurable with respect to $\xi_1, \xi_2, \dots, \xi_{t-1}$. 
	Then, for any fixed $\lambda > 0$ and $\delta \in (0,1)$, with probability at least $1-\delta$, it holds that
	\[
	\sum_{t=1}^T Z_t \leq \frac{3}{4} \lambda \sum_{t=1}^T \sigma_t^2 + \frac{1}{\lambda} \log \frac{1}{\delta}~.
	\]
\end{lemma}
Note that it suffices to pick $\sigma_t \geq \abs{Z_t}$ with probability $1$ for the sub-Gaussian condition to hold.
We also need the following technical lemmas (see \cref{sec:non_convex_appendix} for proofs).
\begin{lemma}\label{lem:sum_eta_g_t_2_affine}
    For \adasgd we have:
    \begin{align*}
        \sum_{t=1}^T \frac{\norm{g_t}^2}{G_t^2}
        \leq
        C_1
        ,
    \end{align*}
    with $C_1$ given in \cref{eq:C1}.
\end{lemma}
\begin{lemma}\label{lem:decorrelated_difference}
    For all $t \geq 1$,
    \begin{align*}
        \sum_{s=1}^t & \abs{\tilde{\eta}_s-\eta_s} (\nabla f(w_s) \dotp g_s)
        \leq
        \frac{\fdiffmax_t}{4}
        + \frac{1}{2} \sum_{s=1}^t  \tilde{\eta}_s \norm{\nabla f(w_s)}^2
        + 2 \eta \sigma_0 \sum_{s=1}^t \frac{\norm{g_s}^2}{G_s^2}
        + 8 \eta^2 \sm \sigma_1^2 \brk*{\sum_{s=1}^t \frac{\norm{g_s}^2}{G_s^2}}^2.
    \end{align*}
    \end{lemma}

Finally, we require a standard property of smooth functions.

\begin{lemma}\label{lem:smooth_norm_bound}
Let $f:\reals^d \mapsto \reals$ be $\sm$-smooth with minimum $f^\star$. Then 
$
    \norm{\nabla f(w)}^2
    \leq
    2\sm\brk{f(w)-f^\star}
$
for all $w \in \reals^d$.
\end{lemma}

\begin{proof}[Proof of \cref{lem:fbound_affine}]
First we handle the randomness of $\brk{g_t}_{t=1}^T$.
Consider the sequence of RVs $\brk{Z_t}_{t=1}^T$ defined by $Z_t = \tilde{\eta}_t \nabla f(w_t) \dotp (\nabla f(w_t)-g_t)$.
Then $\brk{Z_t}_{t=1}^T$ is a martingale difference sequence, as
\begin{align*}
    \condE{Z_t}{t-1}
    &=
    \tilde{\eta}_t \nabla f(w_t) \dotp (\nabla f(w_t)-\condE{g_t}{t-1})
    =
    0
    .
\end{align*}
From Cauchy–Schwarz inequality and the noise model,
$\abs{Z_t}
\leq \tilde{\eta}_t \norm{\nabla f(w_t)} \sqrt{\sigma_0^2 + \sigma_1^2 \norm{\nabla f(w_t)}^2}$. Thus, we obtain from \cref{lem:sub_gaussian} (together with a union bound over $t$) that for any fixed $\lambda>0$, with probability at least $1-\delta$, for all $1 \leq t \leq T$,
\begin{align*}
    \sum_{s=1}^t \tilde{\eta}_s \nabla f(w_s) \dotp (\nabla f(w_s) - g_s)
    &
    \leq
    \frac{3 \lambda}{4} \sum_{s=1}^t \tilde{\eta}_s^2 \norm{\nabla f(w_s)}^2 (\sigma_0^2 + \sigma_1^2 \norm{\nabla f(w_s)}^2)
    + \frac{\log \frac{T}{\delta}}{\lambda}
    .
\end{align*}
We will proceed by induction.
The claim is immediate for $t=1$ and we move to prove it for $t+1$.
Using smoothness,
\begin{align*}
    f(w_{s+1})
    &\leq f(w_{s}) - \eta_s \nabla f(w_s) \dotp g_s + \frac{\sm}{2} \eta_s^2 \norm{g_s}^2.
\end{align*}
Summing the above over $s=1,\ldots,t$,
\begin{align}\label{eq:f_t_f_1}
    f(w_{t+1}) - f(w_{1})
    &\leq
    -\sum_{s=1}^{t} \eta_s \nabla f(w_s) \dotp g_s + \frac{\sm}{2} \sum_{s=1}^{t} \eta_s^2 \norm{g_s}^2.
\end{align}
We move to bound the first term of the right hand side.
\begin{align*}
    - & \sum_{s=1}^t \eta_s \nabla f(w_s) \dotp g_s
    =
    - \sum_{s=1}^t \tilde{\eta}_s \norm{\nabla f(w_s)}^2
    +
    \underbrace{
        \sum_{s=1}^t \tilde{\eta}_s \nabla f(w_s) \dotp (\nabla f(w_s) - g_s)
    }_{(*)}
    +
    \underbrace{
        \sum_{s=1}^t (\tilde{\eta_s}-\eta_s) \nabla f(w_s) \dotp g_s
    }_{(**)}
    .
\end{align*}
We will bound $(*)$ and $(**)$. From \cref{lem:smooth_norm_bound} and our induction, $\norm{\nabla f(w_t)}^2 \leq 2 \sm (f(w_t)-f^\star) \leq 2 \sm \fbound$.
Returning to our martingale bound, with the fact that $\tilde{\eta}_s \leq \eta\Big/\sqrt{\sigma_0^2 + \sigma_1^2 \norm{\nabla f(w_s)}^2}$, we have
\begin{align*}
    \sum_{s=1}^t \tilde{\eta}_s \nabla f(w_s) \dotp (\nabla f(w_s) - g_s)
    &\leq
    \frac{3 \lambda \eta}{4} \sum_{s=1}^t \tilde{\eta}_s \norm{\nabla f(w_s)}^2 \sqrt{\sigma_0^2 + \sigma_1^2 \norm{\nabla f(w_s)}^2}
    + \frac{\log \frac{T}{\delta}}{\lambda}
    \\
    &\leq
    \frac{3 \lambda \eta}{4} \sqrt{\sigma_0^2 + 2 \sm \sigma_1^2 \fbound} \sum_{s=1}^t \tilde{\eta}_s \norm{\nabla f(w_s)}^2 + \frac{\log \frac{T}{\delta}}{\lambda}
    .
\end{align*}
Setting $\lambda = 2 \Big/ 3 \eta \sqrt{\sigma_0^2 + 2 \sm \sigma_1^2 \fbound}$,
\begin{align*}
    \sum_{s=1}^t \tilde{\eta}_s \nabla f(w_s) \dotp (\nabla f(w_s) - g_s)
    &\leq
    \frac{1}{2} \! \sum_{s=1}^t \tilde{\eta}_s \norm{\nabla f(w_s)}^2
    + \frac{3 \eta}{2}\sqrt{\sigma_0^2 + 2 \sm \sigma_1^2 \fbound}\log \tfrac{T}{\delta}
    \\
    &\leq
    \frac{1}{2} \! \sum_{s=1}^t \tilde{\eta}_s \norm{\nabla f(w_s)}^2
    + \frac{3 \eta}{2} \sigma_0 \log \tfrac{T}{\delta}
    + \frac{3 \eta}{2}\sqrt{2 \sm \sigma_1^2 \fbound}\log \tfrac{T}{\delta}
    .
\end{align*}
Using the elementary fact that $ab \leq a^2/2+ b^2/2$ with $a=\sqrt{\fbound/2}$ and $b=3 \eta \sqrt{\sm}\sigma_1\log \tfrac{T}{\delta}$,
\begin{align*}
    \sum_{s=1}^t \tilde{\eta}_s \nabla f(w_s) \dotp (\nabla f(w_s) - g_s)
    &
    \leq
    \frac12 \sum_{s=1}^t \tilde{\eta}_s \norm{\nabla f(w_s)}^2 
    + \frac32 \eta \sigma_0 \log \tfrac{T}{\delta}
    + \frac14\fbound
    + \frac92 \eta^2 \sm \sigma_1^2 \log^2 \tfrac{T}{\delta}
    .
\end{align*}
In order to bound $(**)$, using \cref{lem:decorrelated_difference} with $\fdiffmax_t \leq \fbound$ by induction,
\begin{align*}
    (**)
    &
    \leq
    \frac{\fbound}{4}
    + \frac{1}{2}\sum_{s=1}^t  \tilde{\eta}_s \norm{\nabla f(w_s)}^2
    + 2 \eta \sigma_0 \sum_{s=1}^t \frac{\norm{g_s}^2}{G_s^2}
    + 8 \eta^2 \sm \sigma_1^2 \brk*{\sum_{s=1}^t \frac{\norm{g_s}^2}{G_s^2}}^2
    \\
    &\leq
    \frac{\fbound}{4}
    + \frac{1}{2}\sum_{s=1}^t  \tilde{\eta}_s \norm{\nabla f(w_s)}^2
    + 2 \eta \sigma_0 C_1
    + 8 \eta^2 \sm \sigma_1^2 C_1^2
    ,
\end{align*}
where that last inequality follows from \cref{lem:sum_eta_g_t_2_affine}.
Hence, applying the bounds $(*) \text{ and } (**)$,
\begin{align*}
    - \sum_{s=1}^t \eta_s \nabla f(w_s) \dotp g_s
    &
    \leq
    \tfrac12 \fbound
    + \brk!{\tfrac32 \log \tfrac{T}{\delta} + 2 C_1} \eta \sigma_0
    + \brk!{\tfrac92 \log^2 \tfrac{T}{\delta} + 8 C_1^2} \eta^2 \sm \sigma_1^2
    .
\end{align*}
Thus, returning to \cref{eq:f_t_f_1} and applying \cref{lem:sum_eta_g_t_2_affine},
\begin{align*}
    f(w_{t+1})-f(w_1)
    &
    \leq
    \tfrac12 \fbound
    + \brk!{\tfrac32 \log \tfrac{T}{\delta} + 2 C_1} \eta \sigma_0
    + \brk!{\tfrac92 \log^2 \tfrac{T}{\delta} + 8 C_1^2} \eta^2 \sm \sigma_1^2
    + \frac{\sm}{2} \sum_{s=1}^{t} \eta_s^2 \norm{g_s}^2
    \\
    &
    \leq
    \tfrac12 \fbound
    + \brk!{\tfrac32 \log \tfrac{T}{\delta} + 2 C_1} \eta \sigma_0
    + \brk!{\tfrac92 \log^2 \tfrac{T}{\delta} + 8 C_1^2} \eta^2 \sm \sigma_1^2
    + \tfrac12 \eta^2 \sm C_1
    \\
    &= \fbound - (f(w_1) - f^\star)
    .
\end{align*}
We conclude the induction by adding $f(w_1)-f^\star$ to both sides, obtaining $\fdiffmax_{T+1} = \max_{1 \leq t \leq T+1} \fdiff_t \leq \fbound$.
\end{proof}

%% file: convex_affine.tex
\section{Convergence analysis in the convex case}
\label{sec:convex}
In this section we give a convergence analysis for \adasgd provided that the objective $f$ is a convex and $\sm$-smooth function over $\reals^d$.
We prove the following:
\begin{theorem}\label{thm:convex_convergence_affine}
    Assume that $f$ is convex and $\sm$-smooth.
    Then for
    \adasgd, for any $\delta \in (0,1/4)$,
    it holds with probability at least $1-\delta$ that
    \begin{enumerate}[label=(\roman*)]
        \item%
        $\begin{aligned}[t]
            f\brk3{\frac{1}{T}\sum_{t=1}^T w_t}-f^\star
            &
            \leq
            O\bigg(
                \brk2{\tfrac{\wbound^2}{\eta}+\eta}
                + \wbound \sqrt{\log \tfrac{1}{\delta}}
            \bigg)
            \frac{\sigma_0}{\sqrt{T}}
            \\&
            +
            O\bigg(
                \tfrac{\initg \wnorm_1^2}{\eta}
                + \sm (1+\sigma_1^2) \brk2{\tfrac{\wbound^2}{\eta}+\eta}^2
                + \sm \sigma_1^2 \wbound^2 \log \tfrac{1}{\delta}
            \bigg)
            \frac{1}{T}
            ;
        \end{aligned}$
        \item%
        and $\norm{w_t-w^\star} \leq \wbound$ for all $1 \leq t \leq T+1$,
    \end{enumerate}
    where $\wnorm_1=\norm{w_1-w^\star}$, and
    \begin{align*}
        \wbound^2
        &=
        \tO
        \Big(
            \wnorm_1^2
            +\eta^2
            \big(
                1+\sigma_1^2 + \initg^{-2}\brk{\sigma_0^2+\sm \sigma_1^2 \fbound}
                + \initg^{-4} (\sigma_0^4+\sm^2 \sigma_1^4 \fbound^2)
            \big)
        \Big)
        ,
        \\
        \fbound
        &=
        \widetilde{O}\brk*{
        f(w_1)-f^\star
        + \eta \sigma_0
        + \eta^2 \sm \sigma_1^2
        + \eta^2 \sm
        }
        .
    \end{align*}
\end{theorem}
Similarly to the previous section, the bound applies to the case of bounded affine noise, and we defer the extension to sub-Gaussian noise to \cref{thm:convex_convergence_affine_sub_gaussian} in \cref{sec:sub_gaussian_reduction}.

Few additional remarks are in order.
First, note that when $\sigma_0 \rightarrow 0$ we obtain a rate of order $1/T$, matching the optimal rate for non-stochastic smooth optimization (note that this holds even when $\sigma_1 \gg 0$).
Second, unlike the result in the non-convex case, the bound here contains a polynomial dependence on $1/\initg$ (in the $\wbound$ term). 
This is not an artifact of our analysis: setting a overly small $\initg$ can in fact lead \adasgd to gain a distance of $\Omega(T)$ to a minimizer; see \cref{sec:convex_lower_bound} for more details.
That said, the algorithm is fairly robust to the choice of $\initg$; for example, consider the noisy regime with no affine noise (i.e., $\sigma_0 \gg 0$, $\sigma_1=0$): setting $\eta=\Theta(\norm{w_1-w^\star})$
and any $\gamma$ within the wide range
$\initg \in \brk[s]{\sigma_0,\sigma_0 \sqrt{T}}$
leads to the optimal rate (up to logarithmic factors).

\subsection{Proof of~\texorpdfstring{\cref{thm:convex_convergence_affine}}{Theorem 2}}

In the following, we will require some additional notation.
First, for all $t$ we let
$
    \wnorm_t \eqdef \norm{w_t-w^\star} \text{ and } \wmax_t \eqdef \max_{s \leq t} \norm{w_s-w^\star}.
$
In the proof below, we use a slightly different ``decorrelated step size,'' given by
\begin{align*}
    \etade_s
    \eqdef
    \frac{\eta}{\sqrt{G_{s-1}^2+\norm{\nabla f(w_s)}^2}}
    .
\end{align*}
The proof of the main result is based on a regret bound (\cref{lem:convex_regret_affine}) and a bound $\wnorm_t \leq \wbound$ (\cref{lem:wbound_affine}) which holds with high probability, analogous to \cref{lem:non-convex-regret,lem:fbound_affine} of the previous section. See \cref{sec:convex_appendix} for their statements and proofs.
We now prove our main result in this section.
\begin{proof}[Proof of \cref{thm:convex_convergence_affine}]
With probability at least $1-3\delta$ we have from \cref{lem:wbound_affine} that $\wnorm_t \leq \wbound$ for all $t \leq T+1$.
Using \cref{lem:convex_regret_affine} we have
\begin{align*}
    \sum_{t=1}^T g_t \dotp (w_t-w^\star)
    &\leq
    \brk*{\frac{\wbound^2}{\eta} + \eta} \sqrt{\sum_{t=1}^T \norm{g_t}^2}
    + \frac{\initg \wnorm_1^2}{2 \eta}
    .
\end{align*}
Under the noise assumption, smoothness and convexity,
\begin{align*}
    \norm{g_t}^2
    &\leq
    2 \norm{\nabla f(w_t)}^2 + 2\norm{g_t - \nabla f(w_t)}^2
    \leq
    2 (1+\sigma_1^2) \norm{\nabla f(w_t)}^2 + 2 \sigma_0^2
    \\
    &\leq
    4 \sm (1+\sigma_1^2) \nabla f(w_t) \dotp (w_t-w^\star) + 2 \sigma_0^2
    .
\end{align*}
We obtain
\begin{align*}
    \sum_{t=1}^T g_t \dotp (w_t-w^\star)
    &
    \leq
    \brk*{\frac{\wbound^2}{\eta}+\eta} \sigma_0 \sqrt{2 T}
    + \frac{\initg \wnorm_1^2}{2 \eta}
    +2 \brk*{\frac{\wbound^2}{\eta}+\eta} \sqrt{\sm (1+\sigma_1^2) \sum_{t=1}^T \nabla f(w_t) \dotp (w_t-w^\star)}
    .
\end{align*}
Using the inequality $2 a b \leq a^2+b^2$, we
then have
\begin{align}\label{eq:convex_regret_affine}
    \!
    \sum_{t=1}^T g_t \dotp (w_t-w^\star)
    &
    \leq
    \brk*{\frac{\wbound^2}{\eta}+\eta} \sigma_0 \sqrt{2 T}
    + \frac{\initg \wnorm_1^2}{2 \eta}
    + \frac{1}{4} \sum_{t=1}^T \nabla f(w_t) \dotp (w_t-w^\star)
    + 4 \sm (1+\sigma_1^2) \brk*{\frac{\wbound^2}{\eta}+\eta}^2
    .
\end{align}
By \cref{lem:sub_gaussian} with
$Z_t=(\nabla f(w_t)-g_t) \dotp (w_t-w^\star)$, 
$\abs{Z_t} \leq \sqrt{\sigma_0^2 + \sigma_1^2 \norm{\nabla f(w_t)}^2} \norm{w_t-w^\star}$,
and
\begin{align*}
    \lambda
    &
    =
    \frac{1}{\sigma_0 \wbound \sqrt{T/\log \tfrac{1}{\delta}}+6 \sm \sigma_1^2 \wbound^2}
    ,
\end{align*}
with probability at least $1-4\delta$ (union bound with the event of \cref{lem:wbound_affine}),
\begin{align*}
    \sum_{t=1}^T (\nabla f(w_t)-g_t) \dotp (w_t-w^\star)
    &\leq
    \frac{3 \lambda}{4} \sum_{t=1}^T \brk*{\sigma_0^2 + \sigma_1^2 \norm{\nabla f(w_t)}^2} \norm{w_t-w^\star}^2
    + \frac{1}{\lambda} \log \tfrac{1}{\delta}
    \\
    &\leq
    \sigma_0 \sqrt{T \log \tfrac{1}{\delta}} \brk*{\frac{\wmax_T^2}{\wbound}+ \wbound}
    + \frac{\wmax_T^2 \sum_{t=1}^T \norm{\nabla f(w_t)}^2}{8 \sm \wbound^2}
    + 6 \sm \sigma_1^2 \wbound^2 \log \tfrac{1}{\delta}
    .
\end{align*}
Using the bound $\wmax_T \leq \wbound$, with probability at least $1-4\delta$,
\begin{align*}
    \sum_{t=1}^T (\nabla f(w_t)-g_t) \dotp (w_t-w^\star)
    &\leq
    2 \sigma_0 \wbound \sqrt{T \log \tfrac{1}{\delta}}
    + \frac{1}{8 \sm} \sum_{t=1}^T \norm{\nabla f(w_t)}^2
    + 6 \sm \sigma_1^2 \wbound^2 \log \tfrac{1}{\delta}
    \\
    &\leq
    2 \sigma_0 \wbound \sqrt{T \log \tfrac{1}{\delta}}
    + \frac14 \sum_{t=1}^T \nabla f(w_t) \dotp (w_t-w^\star)
    + 6 \sm \sigma_1^2 \wbound^2 \log \tfrac{1}{\delta}
    ,
\end{align*}
where the last inequality follows from \cref{lem:smooth_norm_bound} and convexity.
Summing with \cref{eq:convex_regret_affine} and rearranging, with probability at least $1-4 \delta$,
\begin{align*}
    \sum_{s=1}^t \nabla f(w_s) \dotp (w_s-w^\star)
    &
    \leq
    8 \sm (1+\sigma_1^2) \brk*{\frac{\wbound^2}{\eta}+\eta}^2
    + \brk*{
        \brk*{\frac{\wbound^2}{\eta}+\eta} \sqrt{8}
        + 4 \wbound \sqrt{\log \tfrac{1}{\delta}}
    } \sigma_0 \sqrt{T}
    + \frac{\initg \wnorm_1^2}{\eta}
    \\
    &
    + 12 \sm \sigma_1^2 \wbound^2 \log \tfrac{1}{\delta}
    .
\end{align*}
A standard application of Jensen's inequality concludes the proof (up to a constant factor in the failure probability, which we subsumed in the big-O notation).
\end{proof}

%% file: bounded_noise_reduction.tex
\newcommand{\xtrunc}{\bar{X}}
\newcommand{\xradius}{r}

\section{A General Reduction from sub-Gaussian to Bounded Noise}
\label{sec:sub_gaussian_reduction}

Given an unbiased gradient oracle one can consider different noise assumptions.
For establishing high probability convergence either sub-Gaussian noise or bounded noise are most commonly used.
Sub-Gaussian noise is less restrictive but can make the analysis cumbersome.
In this section we explain how obtaining guarantees with a bounded oracle can yield the same guarantees with a sub-Gaussian oracle (up to logarithmic factors) with high probability.

Let $X$ be a zero-mean random vector such that for some $\sigma>0$, for all $t > 0$,
\begin{align}\label{eq:sub_gaussian_tail}
    \Pr\brk*{\norm{X} \geq t}
    &\leq
    2 \exp\brk{-\ifrac{t^2}{\sigma^2}}
    .
\end{align}
This definition is equivalent to the sub-Gaussian definition used by \citet{li2020high} and is similar to the one used by \citet{kavis2022high}.

The following lemma is a reduction from $X$ to a new bounded random vector with zero-mean which is equal to $X$ with high probability.

\begin{lemma}\label{lem:truncation}
    For any $\delta \in (0,1)$ there exist a random variable $\xtrunc$ such that:
    \begin{enumerate}[nosep,label=(\roman*)]
        \item $\xtrunc$ is zero-mean: $\E\brk[s]{\xtrunc}=0;$
        \item $\xtrunc$ is equal to $X$ w.h.p.: $\Pr\brk*{\xtrunc=X} \geq 1-\delta;$
        \item $\xtrunc$ is bounded with probability 1: $\Pr\brk!{\norm{\xtrunc}<3 \sigma \sqrt{\ln(4/\delta)}}=1.$
    \end{enumerate}
\end{lemma}
The lemma implies that given a sequence of oracle queries $(w_t)_{t=1}^T$ to a $\sigma$-sub-Gaussian unbiased oracle $g(w)$ (where $w_t$ might depend on the previous queries) there exist a $O(\sigma \sqrt{\log(T/\delta')})$-bounded unbiased oracle $\tilde{g}(w)$ such that with probability at least $1-\delta'$ return the same observations.
This happens by the following procedure (Given a query $w$):
\begin{enumerate}
    \item Let $X=g(w)-\E[g(w)].$
    \item Let $\xtrunc$ be the truncated random vector of \cref{lem:truncation} for $\delta=\ifrac{\delta'}{T}$.
    \item Return $\E[g(w)]+\xtrunc$.
\end{enumerate}
Hence, up to logarithmic factors, one can assume a bounded unbiased oracle and obtain the same guarantee with high probability.
Note that this reduction holds even for an affine noise oracle $g$ where
\begin{align}\label{eq:sub-gaussian-assumption}
    \Pr\brk!{\norm{g(w)-\E[g(w)]} \geq t} \leq 2 \exp\brk{-t^2/\brk{\sigma_0^2+\sigma_1^2 \norm{\E[g(w)]}^2}}
\end{align}
by creating a new gradient oracle $\widetilde{g}(w)$ such that
\begin{align*}
    \norm{\widetilde{g}(w)-\E[\widetilde{g}(w)]}^2
    &\leq
    9 \sigma_0^2 \ln(4T/\delta')
    + 9\sigma_1^2 \ln(4T/\delta')\norm{\E[\widetilde{g}(w)]}^2
    .
\end{align*}
More formally, under a sub-Gaussian affine noise assumption, we derive the following two results using \cref{thm:non_convex_convergence,thm:convex_convergence_affine}:
\begin{theorem}\label{thm:non_convex_convergence_sub_gaussian}
Assume that $f$ is $\sm$-smooth and a sub-Gaussian gradient oracle (\cref{eq:sub-gaussian-assumption}).
Then for \adasgd, for any $\delta \in (0,1/3)$,
it holds with probability at least $1-2 \delta$ that
\begin{enumerate}[label=(\roman*)]
    \item%
    $\begin{aligned}[t]
        \frac{1}{T} \sum_{t=1}^T \norm{\nabla f(w_t)}^2
        &
        \leq
        O\brk*{ \tfrac{2\fbound}{\eta} + \eta\sm }
        \frac{\sigma_0 \sqrt{\log \tfrac{T}{\delta}}}{\sqrt{T}}
        \\&
        + O\brk*{ \tfrac{\initg}{\eta} \fdiff_1 +(1+\sigma_1^2 \log \tfrac{1}{\delta} \log \tfrac{T}{\delta})\brk!{\tfrac{2\fbound}{\eta} + \eta\sm}^2 + \sigma_0^2 \log \tfrac{1}{\delta} \log \tfrac{T}{\delta}} \frac{1}{T}
        ;
    \end{aligned}$
    \item%
    and $f(w_t)-f^\star \leq \fbound$ for all $1 \leq t \leq T+1$,
\end{enumerate}
where $\fdiff_1=f(w_1)-f^\star$ and 
\begin{align}
\label{eq:fbound_sub_gausssian}
\fbound
&=
\begin{aligned}[t]
    2 \fdiff_1
    + \brk!{9 \log \tfrac{T}{\delta} + 12 C_1} \eta \sigma_0 \sqrt{\ln \tfrac{4 T}{\delta}}
    + \brk!{81 \log^2 \tfrac{T}{\delta} + 144 C_1^2} \eta^2 \sm \sigma_1^2 \ln \tfrac{4 T}{\delta}
    + \eta^2 \sm C_1
    ,
\end{aligned}
\\
\label{eq:C1_sub_gaussian}
C_1
&=
\begin{aligned}[t]
    \log
    &
    \bigg(
        1
        \!+\!
        \frac{18 \sigma_0^2 T \ln \tfrac{4 T}{\delta} + (8+72 \sigma_1^2 \ln \tfrac{4 T}{\delta}) (\eta^2 \sm^2 T^3 + \sm \fdiff_1 T)}{\initg^2}
    \bigg)
    .\!
\end{aligned}
\end{align}
\end{theorem}
\begin{theorem}\label{thm:convex_convergence_affine_sub_gaussian}
    Assume that $f$ is convex and $\sm$-smooth and a sub-Gaussian gradient oracle (\cref{eq:sub-gaussian-assumption}).
    Then for
    \adasgd, for any $\delta \in (0,1/4)$,
    it holds with probability at least $1-2 \delta$ that
    \begin{enumerate}[label=(\roman*)]
        \item%
        $\begin{aligned}[t]
            f\brk3{\frac{1}{T}\sum_{t=1}^T w_t}-f^\star
            &
            \leq
            O\bigg(
                \brk2{\tfrac{\wbound^2}{\eta}+\eta}
                + \wbound \sqrt{\log \tfrac{1}{\delta}}
            \bigg)
            \frac{\sigma_0 \sqrt{\log \tfrac{T}{\delta}}}{\sqrt{T}}
            \\&
            +
            O\bigg(
                \tfrac{\initg \wnorm_1^2}{\eta}
                + \sm (1+\sigma_1^2 \log \tfrac{T}{\delta}) \brk2{\tfrac{\wbound^2}{\eta}+\eta}^2
                + \sm \sigma_1^2 \wbound^2 \log \tfrac{1}{\delta} \log \tfrac{T}{\delta}
            \bigg)
            \frac{1}{T}
            ;
        \end{aligned}$
        \item%
        and $\norm{w_t-w^\star} \leq \wbound$ for all $1 \leq t \leq T+1$,
    \end{enumerate}
    where $\wnorm_1=\norm{w_1-w^\star}$, and
    \begin{align*}
        \wbound^2
        &=
        \tO
        \Big(
            \wnorm_1^2
            +\eta^2
            \big(
                1+\sigma_1^2 + \initg^{-2}\brk{\sigma_0^2+\sm \sigma_1^2 \fbound}
                + \initg^{-4} (\sigma_0^4+\sm^2 \sigma_1^4 \fbound^2)
            \big)
        \Big)
        ,
        \\
        \fbound
        &=
        \widetilde{O}\brk*{
        f(w_1)-f^\star
        + \eta \sigma_0
        + \eta^2 \sm \sigma_1^2
        + \eta^2 \sm
        }
        .
    \end{align*}
\end{theorem}
The theorems share the same bounds with \cref{thm:non_convex_convergence,thm:convex_convergence_affine} up to logarithmic factors. Note that a careful modification of the analyses can yield slightly tighter bounds as \cref{lem:sub_gaussian} already holds for a sub-Gaussian martingale difference sequence.
\begin{proof}[Proof of \cref{thm:non_convex_convergence_sub_gaussian,thm:convex_convergence_affine_sub_gaussian}]
    Let $\widetilde{g}(w)$ be the gradient oracle defined above by applying \cref{lem:truncation} to the noise of the sub-Gaussian oracle.
    Due to \cref{lem:truncation} this oracle returns unbiased gradients with bounded affine noise for parameters $\widetilde{\sigma}_0=3 \sigma_0 \sqrt{\ln (\ifrac{4 T}{\delta'})}$ and $\widetilde{\sigma}_1=3 \sigma_1 \sqrt{\ln (\ifrac{4 T}{\delta'})}$. Furthermore, with probability $1-\delta'$, \adasgd performs the exact steps with both oracles as the noisy gradients are equal.
    Hence we can use the guarantees of \cref{thm:non_convex_convergence,thm:convex_convergence_affine}. We conclude by plugging $\widetilde{\sigma}_0$ and $\widetilde{\sigma}_1$ to the bounds and performing a union-bound with the probability the two oracles output the same noisy gradients.
\end{proof}
\subsection{Proof of \texorpdfstring{\cref{lem:truncation}}{Lemma 9}}
\begin{proof}%
Denote $\xradius=\sigma \sqrt{\ln(4/\delta)}$; note that $\Pr(\norm{X} > \xradius) \leq \delta/2$ (due to \cref{eq:sub_gaussian_tail}).
Define a random vector $Z$ by rejection sampling $X$ on the $\xradius$-unit ball: if $\norm{X} \leq \xradius$, let $Z=X$; otherwise, continue resampling from the distribution of $X$ and reject samples as long as their norm is larger than $\xradius$; let $Z$ be the first accepted sample.
To correct the bias of $Z$, let
\begin{align*}
    \xtrunc
    &=
    \begin{cases}
        Z
        &
        \text{with prob.~$1-\delta/2;$}
        \\
        - \frac{2-\delta}{\delta} \E[Z]
        &
        \text{with prob.~$\delta/2.$}
    \end{cases}
\end{align*}
By construction, $\E[\xtrunc]=0$.
Also, $\Pr(\xtrunc=X) \geq 1-\delta$ due to a union bound of $\Pr(\xtrunc=Z)$ and $\Pr(X=Z)$, each happens with probability at least $1-\delta/2$.
We are left with bounding $\xtrunc$.
As $\norm{Z} \leq \xradius$ it remains to bound the correction term.
For this, we first bound $\E[Z]=\E[X \mid \norm{X} \leq \xradius]$.
By the law of total expectation,
\begin{align*}
    \E\brk[s]*{X \mathbf{1}\brk[c]{\norm{X} \leq \xradius}}
    &=
    \Pr(\norm{X} \leq \xradius)\E\brk[s]*{X \mathbf{1}\brk[c]{\norm{X} \leq \xradius} \mid \norm{X} \leq \xradius}
    + \Pr(\norm{X} > \xradius)\E\brk[s]*{X \mathbf{1}\brk[c]{\norm{X} \leq \xradius} \mid \norm{X} > \xradius}
    \\
    &=
    \Pr(\norm{X} \leq \xradius)\E\brk[s]*{X \mathbf{1}\brk[c]{\norm{X} \leq \xradius} \mid \norm{X} \leq \xradius}
    \\
    &=
    \Pr(\norm{X} \leq \xradius)\E\brk[s]*{X \mid \norm{X} \leq \xradius}
    .
\end{align*}
Hence, as $\E[X]=0$,
\begin{align*}
    \E\brk[s]*{X \mid \norm{X} \leq \xradius}
    &=
    \frac{\E\brk[s]*{X \mathbf{1}\brk[c]{\norm{X} \leq \xradius}}}{\Pr(\norm{X} \leq \xradius)}
    =
    -\frac{\E\brk[s]*{X \mathbf{1}\brk[c]{\norm{X} > \xradius}}}{\Pr(\norm{X} \leq \xradius)}
    .
\end{align*}
Thus, using the tail bound (\cref{eq:sub_gaussian_tail}),
\begin{align*}
    \norm{\E[Z]}
    &\leq 
    2 \E[\norm{X} \mathbf{1}\brk[c]{\norm{X} > \xradius}]
    \tag{$\Pr(\norm{X} \leq \xradius)>1/2$}
    \\
    &\leq
    2 \int_0^\infty \Pr(\norm{X} \mathbf{1}\brk[c]{\norm{X} > \xradius} \geq x) dx
    \tag{tail formula for expectations}
    \\
    &\leq
    2 \xradius \Pr(\norm{X} > \xradius)
    + 2 \int_\xradius^\infty \Pr(\norm{X} \mathbf{1}\brk[c]{\norm{X} > \xradius} \geq x) dx
    \tag{split integral at $\xradius$}
    \\
    &\leq
    \delta \xradius
    + 2 \int_\xradius^\infty \Pr(\norm{X} \geq x) dx
    \tag{$\Pr(\norm{X} > \xradius) \leq \delta/2$}
    \\
    &\leq 
    \delta \xradius
    + 4 \int_\xradius^\infty e^{-x^2/\sigma^2} dx
    \tag{\cref{eq:sub_gaussian_tail}}
    \\
    &\leq
    \delta \xradius
    + \frac{2 \sigma^2}{\xradius} \int_\xradius^\infty \frac{2 x}{\sigma^2} e^{-x^2/\sigma^2} dx
    \tag{$x/r \geq 1$ for $x \geq r$}
    \\
    &= 
    \delta \xradius
    -\frac{2 \sigma^2}{\xradius}e^{-x^2/\sigma^2} \big|_\xradius^\infty
    =
    \delta \xradius
    + \frac{2 \sigma^2}{\xradius}e^{-r^2/\sigma^2} 
    \\
    &\leq
    \delta \xradius
    + \frac{\delta \sigma}{2}
    .
\end{align*}
We conclude that $\xtrunc$ is bounded in absolute value (with probability one) by
\begin{align*}
    \max\brk[c]*{r, (\xradius + \sigma/2) (2-\delta)}
    &\leq
    3 \sigma \sqrt{\ln(4/\delta)}
    .
    \qedhere
\end{align*}
\end{proof}

%% file: diameter_lower.tex
\section{Lower Bound for the Effective Diameter in the Convex Case}
\label{sec:convex_lower_bound}

In \cref{sec:convex} we establish a bound of $\max_{t \leq T+1} \norm{w_{t}-w^\star}^2$ which has polynomial dependence in $\ifrac{1}{\initg}$.
In this section we construct a one-dimensional example which shows that such polynomial dependence is necessary in order to obtain a bound of $\max_{t \leq T+1} \norm{w_{t}-w^\star}^2$ which is $o(T)$ with high probability.

Let $f(w)=\tfrac{\sm}{2}w^2$ for some $0 < \sm < \frac{\sigma}{T^{3/2}}$ ($\sm$ is extremely small with respect to $\sigma$).
Let $g(w)$ be a stochastic oracle such that
\begin{align*}
    g(w)
    &=
    \begin{cases}
        \nabla f(w) + \sigma & \text{w.p. $\frac{1}{T}$;}
        \\
        \nabla f(w) - \frac{\sigma}{T-1} & \text{w.p. $1-\frac{1}{T}$,}
    \end{cases}
\end{align*}
where $T > 1$.
Hence, $E\brk[s]{g(w)-\nabla f(w)}=0$ and
$\norm{g(w)-\nabla f(w)}^2 \leq \sigma^2$ with probability $1$.
Let $w_1=w^\star=0$ and $\eta=1$.

With constant probability ($\approx e^{-1}$), $g_t=\nabla f(w_t) - \sigma/(T-1)$ for all $t \in [T]$.
As $\max_{t \leq T+1} \norm{w_{t}-w^\star}^2=o(T)$ with high probability, we assume that for large enough $T$, $\abs{w_t-w^\star} < \frac{\sqrt{T}}{2}$ for all $t \leq T+1$ (still with constant probability under a union bound).
Hence, with constant probability,
\begin{align*}
    w_{T+1}
    &= -\sum_{t=1}^{T} \eta_t g_t
    = -\sum_{t=1}^{T} \eta_t \brk*{\sm w_t-\frac{\sigma}{T-1}}
    \geq \frac{\sigma}{2(T-1)} \sum_{t=1}^{T} \eta_t
    ,
\end{align*}
where the last inequality is the assumption of small $\sm$.
Thus, as $\norm{g_s} = \norm{\sm w_t - \ifrac{\sigma}{(T-1)}} \leq \ifrac{\sigma}{(T-1)}$ (as $w_t \geq 0$ for all $t$ since the stochastic gradients are always negative),
\begin{align*}
    \sum_{t=1}^T \eta_t
    &=
    \sum_{t=1}^T \frac{1}{\sqrt{\initg^2 + \sum_{s=1}^t \norm{g_s}^2}}
    \geq
    \sum_{t=1}^T \frac{1}{\sqrt{\initg^2 + \frac{\sigma^2 t}{(T-1)^2}}}
    =
    \frac{T-1}{\sigma} \sum_{t=1}^T \frac{1}{\sqrt{\frac{\initg^2 (T-1)^2}{\sigma^2} + t}}
    .
\end{align*}
Returning to $w_{T+1}$, bounding sum by integration,
\begin{align*}
    w_{T+1}
    &\geq
    \frac{1}{2} \sum_{t=1}^T \frac{1}{\sqrt{\frac{\initg^2 (T-1)^2}{\sigma^2} + t}}
    \geq
        \sqrt{\frac{\initg^2 (T-1)^2}{\sigma^2} + T+1}-\sqrt{\frac{\initg^2 (T-1)^2}{\sigma^2} + 1}
    =
    \Omega\brk*{\ifrac{\sigma}{\initg}}
    .
\end{align*}
Thus, $\norm{w_{T+1}-w^\star}^2 = \Omega\brk{\ifrac{\sigma^2}{\initg^2}}$.

%% file: known-non-convex.tex
\section{Non-convex SGD with known parameters}
\label{sec:known-non-convex}

In this section we provide a high probability convergence analysis of Stochastic Gradient Descent (SGD) for a $\sm$-smooth objective function $f$ under an affine noise model with known parameters $(\sm,\sigma_0,\sigma_1)$.
The tuned step-size we use is the following,
\begin{align}\label{eq:known-eta}
    \eta
    &=
    \min\brk[c]*{\frac{1}{4 \sm \brk{1 + \sigma_1^2 \log \tfrac{T}{\delta}}},\frac{\knowndiam}{\sigma_0 \sqrt{T}}}
    ,
\end{align}
for some parameter $\knowndiam>0$.
Following is the convergence guarantee.
\begin{theorem}\label{thm:known-non-convex-convergence}
    Assume that $f$ is $\sm$-smooth.
    Then for SGD with step-size (\cref{eq:known-eta}), for any $\delta \in (0,\tfrac{1}{2})$,
    it holds with probability at least $1-2 \delta$ that
    \begin{enumerate}[label=(\roman*)]
        \item%
        $\begin{aligned}[t]
            \frac{1}{T} \! \sum_{t=1}^T \norm{\nabla f(w_t)}^2
            &
            \leq
            \brk2{\sm \knowndiam + \frac{\fdiff_1}{\knowndiam}}\frac{2 \sigma_0 }{\sqrt{T}}
            + \brk*{8 \sm \fdiff_1 (1+ 4 \sigma_1^2) \log \tfrac{T}{\delta}
            + 24 \sigma_1^2 \sm^2 \knowndiam^2 \log \tfrac{1}{\delta}
            + 15 \sigma_0^2 \log \tfrac{1}{\delta}}
            \frac{1}{T}
            ;
        \end{aligned}$
        \item%
        and $f(w_t)-f^\star \leq \fbound$ for all $1 \leq t \leq T+1$,
    \end{enumerate}
    where $\fdiff_1=f(w_1)-f^\star$ and
    \begin{align}
    \label{eq:known-f-bound}
    \fbound
    &=
    \begin{aligned}[t]
        2 \fdiff_1
        + 2 \sm \knowndiam^2
        + 3 \min\brk[c]*{\frac{\sigma_0^2}{4 \sm \brk{1 + \sigma_1^2 \log \tfrac{T}{\delta}}},\frac{\sigma_0 \knowndiam}{\sqrt{T}}} \log \tfrac{T}{\delta}
        .
    \end{aligned}
    \end{align}
\end{theorem}
Similarly to \cref{thm:non_convex_convergence}, \cref{thm:known-non-convex-convergence} establishes a high probability bound of $f(w_t)-f^\star$ for all $t$.
Also note that tuning of $\knowndiam=\Theta(\sqrt{\fdiff_1/\sm})$ yields a rate of $\tO\brk!{\ifrac{\sigma_0 \sqrt{\sm \fdiff_1}}{\sqrt{T}}+\ifrac{(1+\sigma_1^2)\sm\fdiff_1}{T}}$, matching  the results in expectation of \citet{ghadimi2013stochastic} and \citet[Theorem 4.8]{bottou2018optimization} up to few additional logarithmic factors being a result of the high probability analysis.

\subsection{Proof of \texorpdfstring{\cref{thm:known-non-convex-convergence}}{Theorem 5}}

In order to prove \cref{thm:known-non-convex-convergence} the next lemma establish a bound of $f(w_{t})-f^\star$ for all $t \leq T$ that holds with high probability, similarly to \cref{lem:fbound_affine} (proof given later in the section).
\begin{lemma}\label{lem:known-fbound-affine}
    With probability at least $1-\delta$, it holds that
    for all $t \leq T+1$,
    $f(w_{t})-f^\star \leq \fbound$
    (with $\fbound$ defined in
    \cref{thm:known-non-convex-convergence}
    ).
\end{lemma}
We continue to prove the main result using similar lines to the proof of \cref{thm:non_convex_convergence}.
\begin{proof}[Proof of \cref{thm:known-non-convex-convergence}]
    We start with a regret analysis.
    Using $\sm$-smoothness, we have for all $t \geq 1$ that
    \begin{align*}
        f(w_{t+1})
        \leq
        f(w_t) - \eta \nabla f(w_t) \dotp g_t + \tfrac12 \sm \eta^2 \norm{g_t}^2
        ,
    \end{align*}
    which implies
    \begin{align*}
        \nabla f(w_t) \dotp g_t
        &\leq
        \frac{1}{\eta} \brk!{f(w_t) - f(w_{t+1})} 
        + \tfrac12 \sm \eta \norm{g_t}^2
        .
    \end{align*}
    Summing the above over $t=1,\ldots,T$, we obtain (with $\fdiff_1=f(w_1)-f^\star$)
    \begin{align*}
        \sum_{t=1}^T \nabla f(w_t) \dotp g_t
        &\leq
        \frac{\fdiff_1}{\eta} 
        + \frac{\sm \eta}{2} \sum_{t=1}^T \norm{g_t}^2
        .
    \end{align*}
    Observe that
    \begin{align*}
        \norm{g_t}^2
        &\leq
        2\norm{\nabla f(w_t)}^2 + 2\norm{g_t - \nabla f(w_t)}^2
        \leq
        2(1+\sigma_1^2)\norm{\nabla f(w_t)}^2 + 2\sigma_0^2
        .
    \end{align*}
    Hence,
    \begin{align*}
        \sum_{t=1}^T \norm{g_t}^2
        &\leq
        2 \sum_{t=1}^T (1+\sigma_1^2) \norm{\nabla f(w_t)}^2
        + 2 \sigma_0^2 T
        .
    \end{align*}
    By the definition of $\eta$ (\cref{eq:known-eta}),
    \begin{align}\label{eq:known-regret}
        \nonumber
        \sum_{t=1}^T \nabla f(w_t) \dotp g_t
        &\leq
        \frac{\fdiff_1}{\eta} 
        + \sm \eta \sum_{t=1}^T (1+\sigma_1^2) \norm{\nabla f(w_t)}^2
        + \sm \eta \sigma_0^2 T
        \\
        &\leq
        4 \sm \fdiff_1 (1+\sigma_1^2 \log \tfrac{T}{\delta})
        + \frac{1}{4} \sum_{t=1}^T \norm{\nabla f(w_t)}^2
        + \sigma_0 \sqrt{T} \brk2{\sm \knowndiam + \frac{\fdiff_1}{\knowndiam}}
        .
    \end{align}
    We continue with a martingale analysis to replace $\sum_{t=1}^T \nabla f(w_t) \dotp g_t$ with $\sum_{t=1}^T \norm{\nabla f(w_t)}^2$.
    Consider the sequence of R.V.s $(Z_t)_{t=1}^T$ defined by $Z_t = \nabla f(w_t) \dotp (\nabla f(w_t)-g_t)$.
    Then $(Z_t)_{t=1}^T$ is a martingale difference sequence, as
    \begin{align*}
        \condE{Z_t}{t-1}
        &=
        \nabla f(w_t) \dotp (\nabla f(w_t)-\condE{g_t}{t-1})
        =
        0
        .
    \end{align*}
    By Cauchy-Schwarz inequality and the noise assumption,
    \begin{align*}
        \abs{Z_t}
        &\leq
        \norm{\nabla f(w_t)}\norm{\nabla f(w_t)-g_t}
        \leq
        \norm{\nabla f(w_t)} \sqrt{\sigma_0^2 + \sigma_1^2 \norm{\nabla f(w_t)}^2}
        .
    \end{align*}
    Thus, we obtain from \cref{lem:sub_gaussian} with
    $$
    \abs{Z_t} \leq \norm{\nabla f(w_t)} \sqrt{\sigma_0^2 + \sigma_1^2 \norm{\nabla f(w_t)}^2}$$
    and $\lambda=1\big/3(\sigma_0^2 + 2 \sm \sigma_1^2 \fbound)$ that with probability at least $1-\delta$,
    \begin{align*}
        \sum_{t=1}^T \nabla f(w_t) \dotp (\nabla f(w_t)-g_t)
        &\leq
        \frac{\sum_{t=1}^T (\sigma_0^2+\sigma_1^2 \norm{\nabla f(w_t)}^2) \norm{\nabla f(w_t)}^2}{4 (\sigma_0^2 + 2 \sm \sigma_1^2 \fbound)}
        + 3 (\sigma_0^2+2 \sm \sigma_1^2 \fbound) \log(1/\delta)
        .
    \end{align*}
    From \cref{lem:smooth_norm_bound,lem:known-fbound-affine}, with probability at least $1-\delta$, $\norm{\nabla f(w_t)}^2 \leq 2 \sm \fbound$.
    Hence, under a union bound with \cref{lem:known-fbound-affine}, with probability at least $1-2\delta$,
    \begin{align}\label{eq:known-basic-martingale}
        \sum_{t=1}^T \nabla f(w_t) \dotp (\nabla f(w_t)-g_t)
        &
        \leq
        \frac{1}{4} \sum_{t=1}^T \norm{\nabla f(w_t)}^2
        + 3 (\sigma_0^2+2 \sm \sigma_1^2 \fbound) \log \tfrac{1}{\delta}
        .
    \end{align}
    Summing \cref{eq:known-basic-martingale,eq:known-regret} and rearranging,
    \begin{align*}
        \sum_{t=1}^T \norm{\nabla f(w_t)}^2
        &\leq
        8 \sm \fdiff_1 (1+\sigma_1^2 \log \tfrac{T}{\delta})
        + 6 (\sigma_0^2+2 \sm \sigma_1^2 \fbound) \log \tfrac{1}{\delta}
        + 2 \sigma_0 \sqrt{T} \brk2{\sm \knowndiam + \frac{\fdiff_1}{\knowndiam}}
        .
    \end{align*}
    From \cref{lem:known-fbound-affine,eq:known-eta},
    \begin{align*}
        \fbound
        &\leq
        2 \fdiff_1
        + 2 \sm \knowndiam^2
        + \frac{3 \sigma_0^2}{4 \sm \sigma_1^2}
        .
    \end{align*}
    Thus,
    \begin{align*}
        \sum_{t=1}^T \norm{\nabla f(w_t)}^2
        &\leq
        8 \sm \fdiff_1 (1+\sigma_1^2 \log \tfrac{T}{\delta})
        + 24 \sm \sigma_1^2 \brk{\fdiff_1 + \sm \knowndiam^2} \log \tfrac{1}{\delta}
        + 15 \sigma_0^2 \log \tfrac{1}{\delta}
        + 2 \sigma_0 \sqrt{T} \brk2{\sm \knowndiam + \frac{\fdiff_1}{\knowndiam}}
        \\
        &\leq
        8 \sm \fdiff_1 (1+ 4 \sigma_1^2) \log \tfrac{T}{\delta}
        + 24 \sigma_1^2 \sm^2 \knowndiam^2 \log \tfrac{1}{\delta}
        + 15 \sigma_0^2 \log \tfrac{1}{\delta}
        + 2 \sigma_0 \sqrt{T} \brk2{\sm \knowndiam + \frac{\fdiff_1}{\knowndiam}}
        .\qedhere
    \end{align*}
\end{proof}

\subsection{Proof of \texorpdfstring{\cref{lem:known-fbound-affine}}{Lemma 10}}
\begin{proof}
    Consider the sequence of RVs $\brk{Z_t}_{t=1}^T$ defined by $Z_t = \nabla f(w_t) \dotp (\nabla f(w_t)-g_t)$.
    Then $\brk{Z_t}_{t=1}^T$ is a martingale difference sequence, as
    \begin{align*}
        \condE{Z_t}{t-1}
        &=
        \nabla f(w_t) \dotp (\nabla f(w_t)-\condE{g_t}{t-1})
        =
        0
        .
    \end{align*}
    From Cauchy–Schwarz inequality and the noise model,
    $\abs{Z_t}
    \leq \norm{\nabla f(w_t)} \sqrt{\sigma_0^2 + \sigma_1^2 \norm{\nabla f(w_t)}^2}$. Thus, we obtain from \cref{lem:sub_gaussian} (together with a union bound over $t$) that for any fixed $\lambda>0$, with probability at least $1-\delta$, for all $1 \leq t \leq T$,
    \begin{align*}
        \sum_{s=1}^t \nabla f(w_s) \dotp (\nabla f(w_s) - g_s)
        &
        \leq
        \frac{3 \lambda}{4} \sum_{s=1}^t \norm{\nabla f(w_s)}^2 (\sigma_0^2 + \sigma_1^2 \norm{\nabla f(w_s)}^2)
        + \frac{\log \frac{T}{\delta}}{\lambda}
        .
    \end{align*}
    We will proceed by induction.
    The claim is immediate for $t=1$ and we move to prove it for $t+1$.
    Using smoothness,
    \begin{align*}
        f(w_{s+1})
        &\leq f(w_{s}) - \eta \nabla f(w_s) \dotp g_s + \frac{\sm \eta^2}{2} \norm{g_s}^2.
    \end{align*}
    Summing the above over $s=1,\ldots,t$,
    \begin{align}\label{eq:known-f-bound-sum}
        f(w_{t+1}) - f(w_{1})
        &\leq
        - \eta \sum_{s=1}^{t} \nabla f(w_s) \dotp g_s
        + \frac{\sm \eta^2}{2} \sum_{s=1}^{t} \norm{g_s}^2.
    \end{align}
    The second term of the right hand side can be bounded by the definitions of $\eta$ and the noise model,
    \begin{align*}
        \frac{\sm \eta^2}{2} \sum_{s=1}^{t} \norm{g_s}^2
        &\leq
        \sm \eta^2 \sum_{s=1}^{t} \brk!{\norm{g_s-\nabla f(w_s)}^2 + \norm{\nabla f(w_s)}^2}
        \tag{$\norm{a+b}^2 \leq 2 \norm{a}^2+2\norm{b}^2$}
        \\
        &\leq
        \sm \eta^2 \sum_{s=1}^{t} \brk!{(1+\sigma_1^2) \norm{\nabla f(w_s)}^2 + \sigma_0^2}
        \tag{$\norm{\nabla f(w_s)-g_s}^2 \leq \sigma_0^2 + \sigma_1^2 \norm{\nabla f(w_s)}^2$}
        \\
        &\leq
        \sm \knowndiam^2
        + \frac{\eta}{4} \sum_{s=1}^{t} \norm{\nabla f(w_s)}^2
        \tag{\cref{eq:known-eta}}
        .
    \end{align*}
    We move to bound the first term of the right hand side.
    We split the sum to obtain a martingale term,
    \begin{align*}
        - \eta \sum_{s=1}^t \nabla f(w_s) \dotp g_s
        &=
        - \eta \sum_{s=1}^t \norm{\nabla f(w_s)}^2
        +
        \eta \sum_{s=1}^t \nabla f(w_s) \dotp (\nabla f(w_s) - g_s)
        .
    \end{align*}
    From \cref{lem:smooth_norm_bound} and our induction, $\norm{\nabla f(w_t)}^2 \leq 2 \sm (f(w_t)-f^\star) \leq 2 \sm \fbound$.
    Returning to our martingale bound, we have
    \begin{align*}
        \sum_{s=1}^t \nabla f(w_s) \dotp (\nabla f(w_s) - g_s)
        &\leq
        \frac{3 \lambda}{4} \sum_{s=1}^t \norm{\nabla f(w_s)}^2 \brk{\sigma_0^2 + \sigma_1^2 \norm{\nabla f(w_s)}^2}
        + \frac{\log \frac{T}{\delta}}{\lambda}
        \\
        &\leq
        \frac{3 \lambda}{4} \brk{\sigma_0^2 + 2 \sm \sigma_1^2 \fbound} \sum_{s=1}^t \norm{\nabla f(w_s)}^2 + \frac{\log \frac{T}{\delta}}{\lambda}
        .
    \end{align*}
    Setting $\lambda = 1 \Big/ \brk{\sigma_0^2 + 2 \sm \sigma_1^2 \fbound}$,
    \begin{align*}
        \sum_{s=1}^t \nabla f(w_s) \dotp (\nabla f(w_s) - g_s)
        &\leq
        \frac{3}{4} \sum_{s=1}^t \norm{\nabla f(w_s)}^2
        + \brk{\sigma_0^2 + 2 \sm \sigma_1^2 \fbound}\log \tfrac{T}{\delta}
        .
    \end{align*}
    Returning to \cref{eq:known-f-bound-sum} and combining the inequalities,
    \begin{align*}
        f(w_{t+1})-f(w_1)
        &\leq
        \sm \knowndiam^2
        + \eta\brk{\sigma_0^2 + 2 \sm \sigma_1^2 \fbound}\log \tfrac{T}{\delta}
        \\
        &\leq
        \sm \knowndiam^2
        + \eta \sigma_0^2 \log \tfrac{T}{\delta}
        + \frac{1}{2} \fbound
        =
        \fbound - (f(w_1)-f^\star)
        ,
    \end{align*}
    where the last equality follows from the definition of $\eta$. Adding $f(w_1)-f^\star$ to both sides concludes the proof.
\end{proof}

%% file: non_convex_appendix.tex
\section{Proofs of \texorpdfstring{\cref{sec:non-convex}}{Section 3}}
\label{sec:non_convex_appendix}
\subsection{Proof of \texorpdfstring{\cref{lem:sum_eta_g_t}}{Lemma 4}}
\begin{proof}%
    Define $\tilde{G}_t=\sqrt{G_t^2-G_0^2}$.
    Observe that for any $t \geq 1$,
    \begin{align*}
        \norm{g_t}^2
        =
        \sum_{s=1}^t \norm{g_s}^2 - \sum_{s=1}^{t-1} \norm{g_s}^2
        =
        \tilde{G}_{t}^2 - \tilde{G}_{t-1}^2
        =
        \brk!{ \tilde{G}_t + \tilde{G}_{t-1} } 
        \brk!{ \tilde{G}_t - \tilde{G}_{t-1} }
        .
    \end{align*}
    Since $\tilde{G}_t \geq \tilde{G}_{t-1}$, this implies
    \begin{align*}
        \frac{\norm{g_t}^2}{\tilde{G}_t}
        =
        \brk*{\frac{\tilde{G}_t + \tilde{G}_{t-1}}{\tilde{G}_t}} 
        \brk*{ \tilde{G}_{t} - \tilde{G}_{t-1} }
        \leq
        2 \brk*{ \tilde{G}_{t} - \tilde{G}_{t-1} }
        .
    \end{align*}
    Summing this over $t=1,\ldots,T$ and telescoping terms, we obtain
    \begin{align*}
        \sum_{t=1}^T \frac{\norm{g_t}^2}{G_t}
        \leq
        \sum_{t=1}^T \frac{\norm{g_t}^2}{\tilde{G}_t}
        \leq
        2 \tilde{G}_T
        =
        2 \sqrt{\sum_{t=1}^T \norm{g_s}^2}
        ,
    \end{align*}
    where the first inequality follows by $G_t \geq \tilde{G}_t$.
    On the other hand, using the inequality $1-x \leq \log(1/x)$,
    \begin{align*}
        \sum_{t=1}^T \frac{\norm{g_t}^2}{G_t^2}
        &=
        \sum_{t=1}^T \brk*{ 1 - \frac{G_{t-1}^2}{G_{t}^2} } 
        \leq
        \sum_{t=1}^T \log\frac{G_{t}^2}{G_{t-1}^2}
        =
        2 \log\frac{G_T}{G_0}
        .\qedhere
    \end{align*}
\end{proof}
\subsection{Proof of \texorpdfstring{\cref{lem:smooth_norm_bound}}{Lemma 8}}
\begin{proof}%
    Let $w^+ = w-\frac{1}{\sm}\nabla f(w)$. Then from smoothness,
    \begin{align*}
        f(w^+)-f(w)
        &\leq
        \nabla f(w) \dotp \brk{w^+-w}
        +\frac{\sm}{2}\norm{w^+-w}^2
        =
        -\frac{1}{2\sm}\norm{\nabla f(w)}^2
        .
    \end{align*}
    Hence,
    \begin{align*}
        \norm{\nabla f(w)}^2
        &\leq
        2 \sm \brk{f(w)-f(w^+)}
        \leq
        2 \sm \brk{f(w)-f^\star}
        .\qedhere
    \end{align*}
\end{proof}
\subsection{Proof of \texorpdfstring{\cref{lem:sum_eta_g_t_2_affine}}{Lemma 6}}
\begin{proof}%
    Using \cref{lem:sum_eta_g_t},
    \begin{align*}
        \sum_{t=1}^T \frac{\norm{g_t}^2}{G_t^2}
        &\leq 2 \log \frac{G_T}{G_0}.
    \end{align*}
    Since $\norm{g_t}^2 \leq 2 \norm{\nabla f(w_t)}^2 + 2 \norm{g_t-\nabla f(w_t)}^2 \leq 2(1+\sigma_1^2) \norm{\nabla f(w_t)}^2 + 2\sigma_0^2$,
    \begin{align*}
        \sum_{t=1}^T \frac{\norm{g_t}^2}{G_t^2}
        &\leq \log \brk*{G_0^2 + 2 \sigma_0^2 T + 2 (1+\sigma_1^2)\sum_{t=1}^T \norm{\nabla f(w_t)}^2}+2 \log(1/G_0).
    \end{align*}
    From smoothness,
    \begin{align*}
        \norm{\nabla f(w_t)}
        &\leq \norm{\nabla f(w_{t-1})} + \norm{\nabla f(w_{t})-\nabla f(w_{t-1})}
        \leq \norm{\nabla f(w_{t-1})} + \sm \norm{w_t-w_{t-1}}
        \\
        &=
        \norm{\nabla f(w_{t-1})} + \sm \eta_t \norm{g_t}
        \leq \norm{\nabla f(w_{t-1})} + \eta \sm
        .
    \end{align*}
    Thus,
    \begin{align*}
        \norm{\nabla f(w_t)}
        &\leq
        \norm{\nabla f(w_1)}
        + \eta \sm (t-1)
        \leq
        \norm{\nabla f(w_1)}
        + \eta \sm t
        .
    \end{align*}
    summing from $t=1,\ldots,T$,
    \begin{align*}
        \sum_{t=1}^T \norm{\nabla f(w_t)}^2
        &\leq 2 \sum_{t=1}^T (\eta^2 \sm^2 t^2 + \norm{\nabla f(w_1)}^2) 
        \tag{$\norm{x+y}^2 \leq 2 \norm{x}^2 + 2 \norm{y}^2$}
        \\
        &\leq 2 \eta^2 \sm^2 T^3 + 2 \norm{\nabla f(w_1)}^2 T \tag{$\sum_{t=1}^T t^2 \leq T^3$} \\
        &\leq 2 \eta^2 \sm^2 T^3 + 4 \sm (f(w_1)-f^\star)T \tag{from smoothness, \cref{lem:smooth_norm_bound}}
        .
    \end{align*}
    Hence, with $G_0^2 = \initg^2$,
    \begin{align*}
        \sum_{t=1}^T \frac{\norm{g_t}^2}{G_t^2}
        &\leq \log \brk*{1+\frac{2 \sigma_0^2 T + 4 (1+\sigma_1^2) \eta^2 \sm^2 T^3
        + 8 (1+\sigma_1^2) \sm (f(w_1)-f^\star)T}{\initg^2}}
        \leq C_1
        .
        \qedhere
    \end{align*}
\end{proof}

\subsection{Proof of \texorpdfstring{\cref{lem:decorrelated_difference}}{Lemma 7}}
In order to prove the bound we also need the following technical lemma (proof follows).
\begin{lemma}\label{lem:decorrelated_step_size_diff}
    For all $1 \leq s \leq T$,
    \begin{align*}
        \abs{\tilde{\eta}_s-\eta_s}
        &\leq
        \frac{2 \tilde{\eta}_s \sqrt{\sigma_0^2 +\sigma_1^2 \norm{\nabla f(w_s)}^2}}{G_s}
        .
    \end{align*}
\end{lemma}

\begin{proof}[Proof of \cref{lem:decorrelated_difference}]
By \cref{lem:decorrelated_step_size_diff},
\begin{align*}
    \abs{\tilde{\eta}_s-\eta_s}
    &\leq
    \frac{2 \tilde{\eta}_s \sqrt{\sigma_0^2 +\sigma_1^2 \norm{\nabla f(w_s)}^2}}{G_s}
    .
\end{align*}
Hence, by applying Cauchy-Schwarz inequality and using $ab \leq \frac{a^2}{2} + \frac{b^2}{2}$ with $a=\sqrt{\tilde{\eta}_s}\norm{\nabla f(w_s)}$ and $b=2\sqrt{\tilde{\eta}_s(\sigma_0^2+\sigma_1^2 \norm{\nabla f(w_s)}^2)}\frac{\norm{g_s}}{G_s}$,
\begin{align*}
    \abs{\tilde{\eta}_s-\eta_s} \nabla f(w_s) \dotp g_s
    &\leq
    2 \tilde{\eta}_s \norm{\nabla f(w_s)} \sqrt{\sigma_0^2 +\sigma_1^2 \norm{\nabla f(w_s)}^2} \frac{\norm{g_s}}{G_s}
    \\
    &\leq
    \frac{1}{2} \tilde{\eta}_s \norm{\nabla f(w_s)}^2 + 2 \tilde{\eta}_s \brk*{\sigma_0^2 +\sigma_1^2 \norm{\nabla f(w_s)}^2} \frac{\norm{g_s}^2}{G_s^2}
    \\
    &\leq
    \frac{1}{2} \tilde{\eta}_s \norm{\nabla f(w_s)}^2+ 2 \eta \sqrt{\sigma_0^2 +\sigma_1^2 \norm{\nabla f(w_s)}^2} \frac{\norm{g_s}^2}{G_s^2}
    \tag{$\tilde{\eta_s} \leq \eta/\sqrt{\sigma_0^2 + \sigma_1^2 \norm{\nabla f(w_s)}^2}$}
    .
\end{align*}
Using the smoothness of $f$, $\norm{\nabla f(w_s)}^2 \leq 2 \sm (f(w_s)-f^\star)\leq 2 \sm \fdiffmax_t$ (\cref{lem:smooth_norm_bound}). Thus, summing for $s=1,\ldots,t$,
\begin{align*}
    \sum_{s=1}^t \abs{\tilde{\eta}_s-\eta_s} \nabla f(w_s) \dotp g_s
    &\leq
    \frac{1}{2} \sum_{s=1}^t \tilde{\eta}_s \norm{\nabla f(w_s)}^2
    + 2 \eta \sqrt{\sigma_0^2 +2 \sm \sigma_1^2 \fdiffmax_t} \sum_{s=1}^t \frac{\norm{g_s}^2}{G_s^2}
    .
\end{align*}
Again, using $ab \leq \frac{a^2}{2} + \frac{b^2}{2}$ with $a=\sqrt{\fdiffmax_t/2}$ and $b=4\eta \sigma_1\sqrt{\sm}\sum_{s=1}^t \frac{\norm{g_s}^2}{G_s^2}$,
\begin{align*}
    \sum_{s=1}^t \abs{\tilde{\eta}_s-\eta_s} \nabla f(w_s) \dotp g_s
    &\leq
    \frac{1}{2} \sum_{s=1}^t \tilde{\eta}_s \norm{\nabla f(w_s)}^2
    + 2 \eta \sigma_0 \sum_{s=1}^t \frac{\norm{g_s}^2}{G_s^2}
    + 2 \eta \sqrt{2 \sm \sigma_1^2 \fdiffmax_t} \sum_{s=1}^t \frac{\norm{g_s}^2}{G_s^2}
    \\
    &\leq
    \frac{1}{2} \sum_{s=1}^t \tilde{\eta}_s \norm{\nabla f(w_s)}^2
    + 2 \eta \sigma_0 \sum_{s=1}^t \frac{\norm{g_s}^2}{G_s^2}
    + \frac{\fdiffmax_t}{4}
    + 8 \eta^2 \sm \sigma_1^2 \brk*{\sum_{s=1}^t \frac{\norm{g_s}^2}{G_s^2}}^2
    . \qedhere
\end{align*}
\end{proof}
\begin{proof}[Proof of \cref{lem:decorrelated_step_size_diff}]
    Using $\frac{1}{\sqrt{a}}-\frac{1}{\sqrt{b}}=\frac{b-a}{\sqrt{a}\sqrt{b}(\sqrt{a}+\sqrt{b})}$,
    \begin{align*}
        \tilde{\eta}_s-\eta_s
        &=
        \frac{\eta}{\sqrt{G_{s-1}^2 + \sigma_0^2 + (1+\sigma_1^2) \norm{\nabla f(w_s)}^2}}-\frac{\eta}{\sqrt{G_{s-1}^2 + \norm{g_s}^2}}
        \\
        &=
        \frac{\eta(\norm{g_s}^2-\sigma_0^2 -(1+\sigma_1^2) \norm{\nabla f(w_s)}^2)}{G_{s} \sqrt{G_{s-1}^2 + \sigma_0^2 + (1+\sigma_1^2) \norm{\nabla f(w_s)}^2}(G_{s} + \sqrt{G_{s-1}^2 + \sigma_0^2 + (1+\sigma_1^2) \norm{\nabla f(w_s)}^2})}
        \\
        &=
        \frac{\tilde{\eta}_s}{G_s} \cdot \frac{\norm{g_s}^2-\sigma_0^2 -(1+\sigma_1^2) \norm{\nabla f(w_s)}^2}{G_{s} + \sqrt{G_{s-1}^2 + \sigma_0^2 + (1+\sigma_1^2) \norm{\nabla f(w_s)}^2}}
        .
    \end{align*}
    Thus,
    \begin{align*}
        \abs{\tilde{\eta}_s-\eta_s}
        &=
        \frac{\tilde{\eta}_s}{G_s}\cdot\frac{\abs*{\norm{g_s}^2-\sigma_0^2 -(1+\sigma_1^2) \norm{\nabla f(w_s)}^2}}{G_s + \sqrt{G_{s-1}^2 + \sigma_0^2 + (1+\sigma_1^2) \norm{\nabla f(w_s)}^2}}
        \\
        &\leq
        \frac{\tilde{\eta}_s}{G_s}\cdot \frac{\sigma_0^2 + \sigma_1^2 \norm{\nabla f(w_s)}^2 + \norm{g_s-\nabla f(w_s)}\norm{g_s+\nabla f(w_s)}}{G_s + \sqrt{G_{s-1}^2 + \sigma_0^2 + (1+\sigma_1^2) \norm{\nabla f(w_s)}^2}}
        .
    \end{align*}
    Since
    $\norm{g_s+\nabla f(w_s)} \leq \norm{g_s} + \norm{\nabla f(w_s)} \leq G_s + \sqrt{G_{s-1}^2 + \sigma_0^2 + (1+\sigma_1^2) \norm{\nabla f(w_s)}^2}$,
    \begin{align*}
        \abs{\tilde{\eta}_s-\eta_s}
        &\leq
        \frac{\tilde{\eta}_s}{G_s} \cdot \brk*{\sqrt{\sigma_0^2 + \sigma_1^2 \norm{\nabla f(w_s)}^2} + \norm{g_s-\nabla f(w_s)}}
        \leq
        \frac{2 \tilde{\eta}_s \sqrt{\sigma_0^2 +\sigma_1^2 \norm{\nabla f(w_s)}^2}}{G_s}
        .
        \qedhere
    \end{align*}
\end{proof}

\subsection{Proof of \texorpdfstring{\cref{lem:basic_martingale_affine}}{Lemma 3}}
\begin{proof}[Proof of \cref{lem:basic_martingale_affine}]
Consider the sequence of R.V.s $(Z_t)_{t=1}^T$ defined by $Z_t = \nabla f(w_t) \dotp (\nabla f(w_t)-g_t)$.
Then $(Z_t)_{t=1}^T$ is a martingale difference sequence, as
\begin{align*}
    \condE{Z_t}{t-1}
    &=
    \nabla f(w_t) \dotp (\nabla f(w_t)-\condE{g_t}{t-1})
    =
    0
    .
\end{align*}
By Cauchy-Schwarz inequality and the noise assumption,
\begin{align*}
    \abs{Z_t}
    &\leq
    \norm{\nabla f(w_t)}\norm{\nabla f(w_t)-g_t}
    \leq
    \norm{\nabla f(w_t)} \sqrt{\sigma_0^2 + \sigma_1^2 \norm{\nabla f(w_t)}^2}
    .
\end{align*}
Thus, we obtain from \cref{lem:sub_gaussian} with
$$
\abs{Z_t} \leq \norm{\nabla f(w_t)} \sqrt{\sigma_0^2 + \sigma_1^2 \norm{\nabla f(w_t)}^2}$$
and $\lambda=1\big/3(\sigma_0^2 + 2 \sm \sigma_1^2 \fbound)$ that with probability at least $1-\delta$,
\begin{align*}
    \sum_{t=1}^T \nabla f(w_t) \dotp (\nabla f(w_t)-g_t)
    &\leq
    \frac{\sum_{t=1}^T (\sigma_0^2+\sigma_1^2 \norm{\nabla f(w_t)}^2) \norm{\nabla f(w_t)}^2}{4 (\sigma_0^2 + 2 \sm \sigma_1^2 \fbound)}
    + 3 (\sigma_0^2+2 \sm \sigma_1^2 \fbound) \log(1/\delta)
    .
\end{align*}
From \cref{lem:smooth_norm_bound,lem:fbound_affine}, with probability at least $1-\delta$, $\norm{\nabla f(w_t)}^2 \leq 2 \sm \fbound$.
Hence, under a union bound with \cref{lem:fbound_affine}, with probability at least $1-2\delta$,
\begin{align*}
    \sum_{t=1}^T \nabla f(w_t) \dotp (\nabla f(w_t)-g_t)
    &
    \leq
    \frac{1}{4} \sum_{t=1}^T \norm{\nabla f(w_t)}^2
    + 3 (\sigma_0^2+2 \sm \sigma_1^2 \fbound) \log(1/\delta)
    .\qedhere
\end{align*}
\end{proof}

%% file: convex_appendix.tex
\section{Proofs of \texorpdfstring{\cref{sec:convex}}{Section 4}}
\label{sec:convex_appendix}
Below are the statements of the two main lemmas used to prove \cref{thm:convex_convergence_affine} (their proofs follow).
The first is a regret analysis which depends on $\wmax_t$.
\begin{lemma}\label{lem:convex_regret_affine}
    Let $f$ be a convex and $\sm$-smooth function. Then
    \begin{align*}
        \sum_{t=1}^T g_t \dotp (w_t-w^\star)
        \leq
        \brk*{\frac{\wmax_T^2}{\eta} + \eta} \sqrt{\sum_{t=1}^T \norm{g_t}^2}
        + \frac{\initg \wnorm_1^2}{2 \eta}
        .
    \end{align*}
\end{lemma}
Similarly to the non-convex case, we need a high probability bound of $\norm{w_t-w^\star}$ in order to use the regret analysis of \cref{lem:convex_regret_affine}; this is given in our next lemma.
\begin{lemma}\label{lem:wbound_affine}
    With probability at least $1-3 \delta$, it holds that for all $t \leq T+1$,
    \begin{align*}
        \wnorm_t^2
        &\leq
        \wbound^2
        \eqdef
        2 \wnorm_1^2
        + \eta^2 \brk*{
        \frac{1}{2}
        + 2 C_1
        + 8 C^2
        + A_T C
        + B_T \brk*{\frac{\sigma_0}{\initg}+\sigma_1}^2
        }
        ,
    \end{align*}
    where $C=2 \log(1+\affsigma^2 T / 2\initg^2) + 7 \affsigma^2 \log(T/\delta)/\initg^2$, $C_1$ is defined at \cref{lem:sum_eta_g_t_2_affine},
    $A_T=512 \log\brk*{\frac{60 \log^2(6T)}{\delta}}$ and $B_T=512 \log^2\brk*{\frac{60 \log^2(6T)}{\delta}}$.
\end{lemma}
\subsection{Proof of \texorpdfstring{\cref{lem:convex_regret_affine}}{Lemma 12}}
\begin{proof}%
From the standard analysis of SGD, we have for all $t$ that
\begin{align*}
    g_t \dotp (w_t-w^\star)
    &=
    \frac{1}{2\eta_t} \brk!{ \norm{w_t-w^\star}^2 - \norm{w_{t+1}-w^\star}^2 } 
        + \frac{\eta_t}{2} \norm{g_t}^2
    .
\end{align*}
Summing over $t=1,\ldots,T$, we have
\begin{align*}
    \sum_{t=1}^T g_t \dotp (w_t-w^\star)
    &\leq
    \frac{\initg \norm{w_1-w^\star}^2}{2 \eta}
    + \frac12 \sum_{t=1}^T \brk*{\frac{1}{\eta_t} 
    - \frac{1}{\eta_{t-1}}} \norm{w_t-w^\star}^2 
    + \frac12 \sum_{t=1}^T \eta_t \norm{g_t}^2
    .
\end{align*}
Now, since $\eta_t \leq \eta_{t-1}$, we have for all $t$ that
\begin{align*}
    \frac{1}{\eta_t} - \frac{1}{\eta_{t-1}}
    \leq
    \eta_t \brk*{ \frac{1}{\eta_t^2} - \frac{1}{\eta_{t-1}^2} } 
    =
    \frac{\eta_t}{\eta^2} \brk*{\initg^2 + \sum_{s=1}^t \norm{g_s}^2 - \initg^2 - \sum_{s=1}^{t-1} \norm{g_s}^2}
    =
    \frac{\eta_t}{\eta^2} \norm{g_t}^2 
    ,
\end{align*}
thus
\begin{align*}
    \sum_{t=1}^T g_t \dotp (w_t-w^\star)
    &\leq
    \frac{\initg \norm{w_1-w^\star}^2}{2 \eta}
    + \frac12 \sum_{t=1}^T \eta_t \norm{g_t}^2 \brk3{1 + \frac{\norm{w_t-w^\star}^2}{\eta^2}}
    .
\end{align*}
Finally, bounding $\norm{w_t-w^\star} \leq \wmax_T$ and applying \cref{lem:sum_eta_g_t}
we obtain
\begin{align*}
    \sum_{t=1}^T g_t \dotp (w_t-w^\star)
    &\leq
    \brk*{\frac{\wmax_T^2}{\eta} + \eta} \sqrt{\sum_{t=1}^T \norm{g_t}^2}
    + \frac{\initg \norm{w_1-w^\star}^2}{2 \eta}
    .\qedhere
\end{align*}
\end{proof}
\subsection{Proof of \texorpdfstring{\cref{lem:wbound_affine}}{Lemma 13}}
In order to obtain the high probability bound of $\wmax_T$ we use of the following empirical Bernstein concentration bound.
\begin{lemma}[Corollary 1 of \citet{carmon2022making}]\label{lem:emp_bernstein}
    Let $X_t$ be adapted to $\cF_t$ such that $\abs{X_t} \leq 1$ with probability $1$ for all $t$.
    Then, for every $\delta \in (0,1)$ and any $\hat{X}_t \in \cF_{t-1}$ such that $\abs{\hat{X}_t}\leq 1$ with probability $1$,
    \begin{align*}
        \Pr\brk*{\exists t < \infty : \abs*{\sum_{s \leq t} (X_s-\E\brk[s]{X_s | \cF_{s-1}})}
        \geq \sqrt{A_t(\delta) \sum_{s \leq t} (X_s-\hat{X}_s)^2 + B_t(\delta)}}
        &\leq
        \delta
        ,
    \end{align*}
    where $A_t(\delta)=16 \log\brk*{\frac{60 \log(6t)}{\delta}}$ and $B_t(\delta)=16 \log^2\brk*{\frac{60 \log(6t)}{\delta}}$.
\end{lemma}
We use the concentration bound to prove the following technical martingale bound.
\begin{lemma}\label{lem:convex_martingale_affine}
    Let $\wmaxeta_t = \max\brk[c]{\wmax_t,\eta}$. Then with probability at least $1-\delta$, it holds that for all $t \leq T$,
    \begin{align*}
        \sum_{s=1}^t & \etade_{s} (\nabla f(w_s)-g_s) \dotp (w_s-w^\star)
        \\
        &\leq
        2 \wmaxeta_t
        \sqrt{A_t(\ifrac{\delta}{\log(4T)}) \sum_{s \leq t} \eta_{s-1}^2 \norm{\nabla f(w_s)-g_s}^2 + \eta^2 \brk*{\frac{\sigma_0}{\initg}+\sigma_1}^2 B_t(\ifrac{\delta}{\log(4T)})}
        ,
    \end{align*}
    where $A(\cdot)$ and $B(\cdot)$ are defined at \cref{lem:emp_bernstein}.
\end{lemma}
The next lemma bounds the $\sum_{s \leq t} \eta_{s-1}^2 \norm{\nabla f(w_s)-g_s}^2$ term which appear in \cref{lem:convex_martingale_affine}.
\begin{lemma}\label{lem:sum_eta_noise}
    With probability at least $1-2\delta$, for all $1 \leq t \leq T$,
    \begin{align*}
            \sum_{s=1}^t \eta_{s-1}^2 \norm{\nabla f(w_s)-g_s}^2
            \leq
            2 \eta^2 \log(1 + \affsigma^2 T / 2 \initg^2)
            + 7 \eta_0^2 \affsigma^2 \log(T/\delta)
        .
    \end{align*}
\end{lemma}
Now we move to prove the lemma.
\begin{proof}[Proof of \cref{lem:wbound_affine}]
    Rolling a single step of SGD,
    \begin{align*}
        \norm{w_{s+1}-w^\star}^2
        &=
        \norm{w_{s}-w^\star}^2
        - 2 \eta_s g_s \dotp (w_s-w^\star)
        + \eta_s^2 \norm{g_s}^2
        .
    \end{align*}
    Summing for $s=1 \ldots t$ and applying \cref{lem:sum_eta_g_t_2_affine},
    \begin{align}\label{eq:roll_convex_sgd}
        \norm{w_{t+1}-w^\star}^2
        &=
        \norm{w_{1}-w^\star}^2
        - 2 \sum_{s=1}^t \eta_s g_s \dotp (w_s-w^\star)
        + \sum_{s=1}^t \eta_s^2 \norm{g_s}^2
        \nonumber
        \\
        &\leq
        \norm{w_{1}-w^\star}^2
        \underbrace{- 2 \sum_{s=1}^t \eta_s g_s \dotp (w_s-w^\star)}_{(*)}
        + \eta^2 C_1
        .
    \end{align}
    Focusing on $(*)$, due to convexity,
    \begin{align*}
        - 2 \sum_{s=1}^t \eta_s g_s \dotp (w_s-w^\star)
        &=
        - 2 \sum_{s=1}^t \eta_s \nabla f(w_s) \dotp (w_s-w^\star)
        + 2 \sum_{s=1}^t \eta_s (\nabla f(w_s)-g_s) \dotp (w_s-w^\star)
        \\
        &\leq
        2 \sum_{s=1}^t \eta_s (\nabla f(w_s)-g_s) \dotp (w_s-w^\star)
        .
    \end{align*}
    In order to create a martingale we replace $\eta_s=\ifrac{\eta}{\sqrt{G_{s-1}^2+\norm{g_s}^2}}$ with $\etade_s = \ifrac{\eta}{\sqrt{G_{s-1}^2+\norm{\nabla f(w_s)}^2}}$,
    \begin{align*}
        (*)
        &\leq
        2 \sum_{s=1}^t \etade_{s} (\nabla f(w_s)-g_s) \dotp (w_s-w^\star)
        + 2 \sum_{s=1}^t (\eta_s-\etade_{s}) (\nabla f(w_s)-g_s) \dotp (w_s-w^\star)
        .
    \end{align*}
    We observe that
    \begin{align*}
        \abs{\eta_s-\etade_s}
        &=
        \eta \frac{\abs2{\sqrt{G_{s-1}^2+\norm{\nabla f(w_s)}^2}-\sqrt{G_{s-1}^2+\norm{g_s}^2}}}{\sqrt{G_{s-1}^2+\norm{g_s}^2}\sqrt{G_{s-1}^2+\norm{\nabla f(w_s)}^2}}
        \\
        &=
        \eta \frac{\abs!{\norm{\nabla f(w_s)}^2-\norm{g_s}^2}}{\sqrt{G_{s-1}^2+\norm{g_s}^2}\sqrt{G_{s-1}^2+\norm{\nabla f(w_s)}^2}(\sqrt{G_{s-1}^2+\norm{g_s}^2}+\sqrt{G_{s-1}^2+\norm{\nabla f(w_s)}^2})}
        \\
        &\leq
        \eta \frac{\norm{\nabla f(w_s)-g_s}(\norm{\nabla f(w_s)}+\norm{g_s})}{\sqrt{G_{s-1}^2+\norm{g_s}^2}\sqrt{G_{s-1}^2+\norm{\nabla f(w_s)}^2}(\sqrt{G_{s-1}^2+\norm{g_s}^2}+\sqrt{G_{s-1}^2+\norm{\nabla f(w_s)}^2})}
        \\
        &\leq
        \eta \frac{\norm{\nabla f(w_s)-g_s}}{\sqrt{G_{s-1}^2+\norm{g_s}^2}\sqrt{G_{s-1}^2+\norm{\nabla f(w_s)}^2}}
        .
    \end{align*}
    Thus
    \begin{align*}
        \sum_{s=1}^{t} (\eta_s-\etade_s)(\nabla f(w_s)-g_s) \dotp (w_s-w^\star)
        &\leq
        \sum_{s=1}^{t} \abs{\eta_s-\etade_s}\norm{\nabla f(w_s)-g_s}\wnorm_s
        \\
        &\leq
        \wmax_t \sum_{s=1}^{t} \abs{\eta_s-\etade_s}\norm{\nabla f(w_s)-g_s}
        \\
        &\leq
        \eta \wmax_t \sum_{s=1}^t \frac{\norm{\nabla f(w_s)-g_s}^2}{\sqrt{G_{s-1}^2+\norm{g_s}^2}\sqrt{G_{s-1}^2+\norm{\nabla f(w_s)}^2}}
        \\
        &\leq \frac{\wmax_t}{\eta} \sum_{s=1}^t \eta_{s-1}^2 \norm{\nabla f(w_s)-g_s}^2
        .
    \end{align*}
    Combining the inequality with \cref{lem:convex_martingale_affine}, with probability at least $1-\delta$,
    \begin{align*}
        (*)
        &\leq
        \frac{2 \wmax_t}{\eta} \sum_{s=1}^t \eta_{s-1}^2 \norm{\nabla f(w_s)-g_s}^2
        \\
        &+ 4 \wmaxeta_t
            \sqrt{A_t(\ifrac{\delta}{\log(4T)}) \sum_{s \leq t} \eta_{s-1}^2 \norm{\nabla f(w_s)-g_s}^2 + \eta^2 \brk*{\frac{\sigma_0}{\initg}+\sigma_1}^2 B_t(\ifrac{\delta}{\log(4T)})}
        ,
    \end{align*}
    and using $ab \leq a^2/2 + b^2/2$,
    \begin{align*}
        (*)
        &\leq
        \frac{\wmax_t^2}{4}
        + \frac{4}{\eta^2} \brk*{\sum_{s=1}^t \eta_{s-1}^2 \norm{\nabla f(w_s)-g_s}^2}^2
        + \frac{\wmaxeta_t^2}{4}
        \\
        &+ 16
            \brk*{A_t(\ifrac{\delta}{\log(4T)}) \sum_{s \leq t} \eta_{s-1}^2 \norm{\nabla f(w_s)-g_s}^2 + \eta^2 \brk*{\frac{\sigma_0}{\initg}+\sigma_1}^2 B_t(\ifrac{\delta}{\log(4T)})}
        \\
        &\leq
        \frac{\wmax_t^2}{2}
        + \frac{4}{\eta^2} \brk*{\sum_{s=1}^t \eta_{s-1}^2 \norm{\nabla f(w_s)-g_s}^2}^2
        + \frac{\eta^2}{4}
        \tag{$\wmaxeta_t=\max\brk[c]{\wmax_t,\eta}$}
        \\
        &+ 16
            \brk*{A_t(\ifrac{\delta}{\log(4T)}) \sum_{s \leq t} \eta_{s-1}^2 \norm{\nabla f(w_s)-g_s}^2 + \eta^2 \brk*{\frac{\sigma_0}{\initg}+\sigma_1}^2 B_t(\ifrac{\delta}{\log(4T)})}
        .
    \end{align*}
    Applying \cref{lem:sum_eta_noise}, with probability at least $1-3\delta$ (via a union bound),
    \begin{align*}
        (*)
        &\leq
        \frac{\wmax_t^2}{2}
        + 4 \eta^2 C^2
        + \frac{\eta^2}{4}
        + 16 \eta^2
            \brk*{A_t(\ifrac{\delta}{\log(4T)}) C + \brk*{\frac{\sigma_0}{\initg}+\sigma_1}^2 B_t(\ifrac{\delta}{\log(4T)})}
        ,
    \end{align*}
    where $C=2 \log(1+\affsigma^2 T / 2\initg^2) + 7 \affsigma^2/\initg^2 \log(T/\delta)$.
    Hence, returning to \cref{eq:roll_convex_sgd} and applying \cref{lem:sum_eta_g_t_2_affine},
    \begin{align*}
        \wnorm_{t+1}^2
        &\leq
        \frac{\wmax_t^2}{2}
        + \wnorm_1^2
        + \eta^2 \brk*{
            \frac{1}{4}
            + C_1
            + 4 C^2
            + 16
            \brk*{A_t(\ifrac{\delta}{\log(4T)}) C + \brk*{\frac{\sigma_0}{\initg}+\sigma_1}^2 B_t(\ifrac{\delta}{\log(4T)})}
        }
        \\
        &\leq
        \frac{\wmax_t^2}{2}
        + \frac{\wbound}{2}
        .
        \tag{def of $\wbound$, $32 A_t(\ifrac{\delta}{\log(4T)}) \leq A_T$ and $32 B_t(\ifrac{\delta}{\log(4T)}) \leq B_T$}
    \end{align*}
    Hence, rolling to $t=1$, with probability at least $1-3\delta$, $\wnorm_t \leq \wbound$ for all $1 \leq t \leq T+1$.
\end{proof}
\begin{proof}[Proof of \cref{lem:convex_martingale_affine}]
    In order to invoke \cref{lem:emp_bernstein} we will replace $w_s-w^\star$ with a version which is scaled and projected to the unit ball.
    We denote
    $
        a_k = 2^{k-1} \wmaxeta_1
        \text{ and }
        k_t=\ceil{\log (\wmaxeta_t/\wmaxeta_1)}+1
        .
    $
    Thus, $\wmax_t \leq \wmaxeta_t \leq a_{k_t} \leq 2\wmaxeta_t$.
    Since $\norm{w_{s+1}-w^\star} \leq \norm{w_{s}-w^\star} + \eta$ for all $s$, $\wmax_t \leq \wnorm_1 + \eta (t-1)$ and $1 \leq k_t \leq \ceil{\log (t)} + 1 \leq \log(4T)$.
    Defining the projection to the unit ball, $\proj{x}=x/\max\brk[c]*{1,\norm{x}}$,
    \begin{align*}
        \frac{w_s-w^\star}{a_{k_t}}
        &=
        \Pi_1 \brk*{\frac{w_s-w^\star}{a_{k_t}}}
        .
    \end{align*}
    To simplify notation, let $C = \eta(\ifrac{\sigma_0}{\initg}+\sigma_1)$ and note that
    \begin{align*}
        \norm{\etade_s (\nabla f(w_s)-g_s)}
        &\leq
        \frac{\eta \sqrt{\sigma_0^2 + \sigma_1^2 \norm{\nabla f(w_s)}^2}}{\sqrt{\initg^2 + \norm{\nabla f(w_s)}^2}}
        \leq
        \frac{\eta \sigma_0}{\initg}
        + \eta \sigma_1
        =
        C
        .
    \end{align*}
    Thus,
    \begin{align}\label{eq:sum_martingales}
        \sum_{s=1}^t & \frac{\etade_s (\nabla f(w_s)-g_s) \dotp (w_s-w^\star)}{C a_{k_t}}
        \nonumber
        \\&=
        \sum_{s=1}^t \frac{\etade_s (\nabla f(w_s)-g_s)}{C} \dotp \Pi_1 \brk*{\frac{w_s-w^\star}{a_{k_t}}}
        \nonumber
        \\
        &\leq
        \abs*{\sum_{s=1}^t \frac{\etade_s (\nabla f(w_s)-g_s)}{C} \dotp \Pi_1 \brk*{\frac{w_s-w^\star}{a_{k_t}}}}
        \nonumber
        \\
        &\leq
        \max_{1 \leq k \leq \floor{\log(4T)}}\abs*{\sum_{s=1}^t \frac{\etade_s (\nabla f(w_s)-g_s)}{C} \dotp \Pi_1 \brk*{\frac{w_s-w^\star}{a_{k}}}}
        .
    \end{align}
    Let $X_s^{(k)}=\frac{\etade_s (\nabla f(w_s)-g_s)}{C} \dotp \Pi_1 \brk*{\frac{w_s-w^\star}{a_{k}}}$ for some $k$.
    $$
    \condE{g_s}{s-1}=\nabla f(w_s)
    \quad\implies\quad
    \condE{X_s^{(k)}}{s-1}=0
    .
    $$
    Also note that $X_s^{(k)}\leq 1$ with probability $1$.
    Using \cref{lem:emp_bernstein} with the $X_s^{(k)}$ we defined and $\hat{X_s}=0$, for any $k$ and $\delta' \in (0,1)$, with probability at least $1-\delta'$, for all $t \leq T$,
    \begin{align*}
        \abs*{\sum_{s \leq t} X_s^{(k)}}
        &\leq
        \sqrt{A_t(\delta') \sum_{s \leq t} (X_s^{(k)})^2 + B_t(\delta')}
        .
    \end{align*}
    We can upper bound $(X_s^{(k)})^2$,
    \begin{align*}
        (X_s^{(k)})^2
        &\leq
        \frac{\etade_s^2 \norm{\nabla f(w_s)-g_s}^2}{C^2} \norm*{\Pi_1 \brk*{\frac{w_s-w^\star}{a_{k}}}}^2
        \tag{Cauchy-Schwarz inequality}
        \\
        &\leq
        \frac{\etade_s^2 \norm{\nabla f(w_s)-g_s}^2}{C^2}
        \tag{$\norm{\Pi_1(x)} \leq 1$}
        \\
        &\leq
        \frac{\eta_{s-1}^2}{C^2} \norm{\nabla f(w_s)-g_s}^2
        \tag{$\eta_{s-1} \geq \etade_s$}
        .
    \end{align*}
    Thus, returning to \cref{eq:sum_martingales} multiplied by $C a_{k_t}$, with probability at least $1-\delta' \log(4T)$ (union bound for all
    $
    1 
    \leq k 
    \leq 
    \floor{\log(4T)})
    ,
    $
    \begin{align*}
        \sum_{s=1}^t
        \etade_s (\nabla f(w_s)-g_s) \dotp (w_s-w^\star)
        &\leq
        C a_{k_t}
        \sqrt{A_t(\delta') \sum_{s \leq t} \frac{\eta_{s-1}^2}{C^2} \norm{\nabla f(w_s)-g_s}^2 + B_t(\delta')}
        .
    \end{align*}
    As $a_{k_t} \leq 2\wmaxeta_t$, picking $\delta'=\ifrac{\delta}{\log(4T)}$, with probability at least $1-\delta$,
    \begin{align*}
        \sum_{s=1}^t
        \eta_{s-1} (\nabla f(w_s)-g_s) \dotp (w_s-w^\star)
        &\leq
        2 \wmaxeta_t
        \sqrt{A_t(\ifrac{\delta}{\log(4T)}) \sum_{s \leq t} \eta_{s-1}^2 \norm{\nabla f(w_s)-g_s}^2 + C^2 B_t(\ifrac{\delta}{\log(4T)})}
        .
        \qedhere
    \end{align*}
\end{proof}
\begin{proof}[Proof of \cref{lem:sum_eta_noise}]
    Let $Z_s = \nabla f(w_s) \dotp (\nabla f(w_s)-g_s)$
    and
    note that $\condE{Z_s}{s-1}=0$. Hence, $\brk{Z_s}_{s=1}^t$ is a martingale difference sequence.
    By \cref{lem:sub_gaussian} (with $\abs{Z_s} \leq \sqrt{\sigma_0^2 + \sigma_1^2 \norm{\nabla f(w_s)}^2} \norm{\nabla f(w_s)}$ and $\lambda=2/3 \affsigma^2$) and a union bound (over $1 \leq t \leq T$), with probability at least $1-\delta$,
    \begin{align*}
        \sum_{s=1}^t \nabla f(w_s) \dotp (\nabla f(w_s)-g_s)
        &\leq
        \frac{3 \lambda}{4} \sum_{s=1}^t (\sigma_0^2 + \sigma_1^2 \norm{\nabla f(w_s)}^2) \norm{\nabla f(w_s)}^2
        + \frac{1}{\lambda} \log(T/\delta)
        \\
        &
        \leq
        \frac{\sigma_0^2 + 2 \sm \sigma_1^2 \fdiffmax_t}{2\affsigma^2} \sum_{s=1}^t \norm{\nabla f(w_s)}^2
        + \frac{3}{2} \affsigma^2 \log(T/\delta)
        \tag{\cref{lem:smooth_norm_bound}}
    \end{align*}
    for all $t \in [T]$, and under a union bound with \cref{lem:fbound_affine}, with probability at least $1-2\delta$, since $\sigma_0^2 + 2 \sm \sigma_1^2 \fdiffmax_t \leq \affsigma^2$,
    \begin{align}\label{eq:convex_martingale_1_affine}
        \sum_{s=1}^t \nabla f(w_s) \dotp (\nabla f(w_s)-g_s)
        &\leq
        \frac{1}{2} \sum_{s=1}^t \norm{\nabla f(w_s)}^2
        + \frac{3}{2} \affsigma^2 \log(T/\delta)
        .
    \end{align}
    Let
    $$
    k = \min\set*{t \leq T : \sum_{s=1}^t \norm{\nabla f(w_s)-g_s}^2 > 6 \affsigma^2 \log(T/\delta)}.
    $$
    For $t < k$,
    \begin{align*}
        \sum_{s=1}^t \eta_{s-1}^2 \norm{\nabla f(w_s)-g_t}^2
        &\leq
        \eta_0^2 \sum_{s=1}^t \norm{\nabla f(w_s)-g_t}^2
        \leq
        6 \eta_0^2 \affsigma^2 \log(T/\delta)
        .
    \end{align*}
    For $t \geq k$,
    \begin{align*}
        G_t^2
        &=
        \initg^2
        + \sum_{s=1}^t \norm{\nabla f(w_s)}^2
        + \sum_{s=1}^t \norm{\nabla f(w_s) - g_s}^2
        - 2 \sum_{s=1}^t \nabla f(w_s) \dotp (\nabla f(w_s)-g_s)
        \\
        &\geq
        \initg^2
        + \sum_{s=1}^t \norm{\nabla f(w_s) - g_s}^2
        - 3 \affsigma^2 \log(T/\delta)
        \tag{\cref{eq:convex_martingale_1_affine}}
        \\
        &\geq
        \initg^2
        + \frac{1}{2} \sum_{s=1}^t \norm{\nabla f(w_s) - g_s}^2
        \tag{$t \geq k$}
        .
    \end{align*}
    Hence,
    \begin{align*}
        \sum_{s=k}^t \eta_{s-1}^2 \norm{\nabla f(w_s)-g_s}^2
        &=
        \sum_{s=k}^t \eta_{s}^2 \norm{\nabla f(w_s)-g_s}^2
        +
        \sum_{s=k}^t (\eta_{s-1}^2-\eta_{s}^2) \norm{\nabla f(w_s)-g_s}^2
        \\
        &\leq
        \sum_{s=k}^t \eta_{s}^2 \norm{\nabla f(w_s)-g_s}^2
        +
        \eta_0^2 \affsigma^2
        \\
        &\leq
        2 \eta^2 \sum_{s=k}^t \frac{\norm{\nabla f(w_s)-g_s}^2}{2 \initg^2 + 
        \sum_{k=1}^s \norm{\nabla f(w_k)-g_k}^2}
        + \eta_0^2 \affsigma^2
        .
    \end{align*}
    Using \cref{lem:sum_eta_g_t} with $G_0=\sqrt{2 \initg^2} \text{ and } g_t=\nabla f(w_s)-g_s,$
    \begin{align*}
        \sum_{s=k}^t \eta_{s-1}^2 \norm{\nabla f(w_s)-g_s}^2
        &\leq
        2 \eta^2 \log\brk*{\frac{2 \initg^2 + \sum_{s=k}^t \norm{\nabla f(w_s)-g_s}^2}{2 \initg^2}}
        + \eta_0^2 \affsigma^2
        \\
        &\leq
        2 \eta^2 \log\brk*{1+\frac{\sigma_0^2 T + \sigma_1^2 \sum_{s=k}^t \norm{\nabla f(w_s)}^2}{2 \initg^2}}
        + \eta_0^2 \affsigma^2
        \tag{noise bound}
        \\
        &\leq
        2 \eta^2 \log(1 + \affsigma^2 T / 2 \initg^2)
        + \eta_0^2 \affsigma^2
        \tag{\cref{lem:smooth_norm_bound,lem:fbound_affine}}
        .
    \end{align*}
    Hence, combining with the case of $t < k$, with probability at least $1-2\delta$, for all $t \leq T$,
    \begin{align*}
        \sum_{s=1}^t \eta_{s-1}^2 \norm{\nabla f(w_s)-g_s}^2
        &\leq
        2 \eta^2 \log(1 + \affsigma^2 T / 2 \initg^2)
        + 7 \eta_0^2 \affsigma^2 \log(T/\delta)
        .
        \qedhere
    \end{align*}
\end{proof}